%% file: main.tex
\def\isarxiv{1}

\ifdefined\isarxiv
\documentclass[12pt]{article}
\usepackage[margin=1.25in]{geometry} 
\usepackage{times}
\usepackage{algpseudocode}
\usepackage{hyperref} 
\else
\documentclass{article}%[nonacm,anonymous,acmsmall,pdftex]{acmart}
\usepackage{hyperref}       % hyperlinks
\usepackage{multicol,enumitem}

\usepackage{icml2026}
\fi

\usepackage{natbib}
\usepackage{url}   

\usepackage{amsmath}
\usepackage{amssymb}
\usepackage{amsthm}

\usepackage{cleveref}
\usepackage{comment}
\usepackage{bbm}
\usepackage{textcomp}
\usepackage{wrapfig}
\usepackage{float}
\usepackage{wrapfig}

\usepackage{url}            % simple URL typesetting
\usepackage{booktabs}       % professional-quality tables
\usepackage{nicefrac}       % compact symbols for 1/2, etc.
\usepackage{microtype}      % microtypography
\usepackage{multirow}
\usepackage{color}
\usepackage[english]{babel}
\usepackage{graphicx}
\usepackage{grffile}
\usepackage{wrapfig,epsfig}
\usepackage{bbm}
\usepackage{comment}
\usepackage{subcaption} % don't add this packge
\usepackage{tikz}
\usepackage{colortbl}
\usepackage{nicematrix}
\usepackage{makecell}

\usepackage{pifont} % render ding marks
\usepackage{pgfplots}
\pgfplotsset{compat=1.8}
\tikzset{elegant/.style={smooth,thick,samples=500,magenta}}

\theoremstyle{plain}
\newtheorem{theorem}{Theorem}[section]
\newtheorem{lemma}[theorem]{Lemma}
\newtheorem{remark}[theorem]{Remark}
\newtheorem{corollary}[theorem]{Corollary}
\newtheorem{proposition}[theorem]{Proposition}
\theoremstyle{definition}
\newtheorem{definition}[theorem]{Definition}

\newtheorem{problem}[theorem]{Problem}

\crefname{assumption}{Assumption}{Assumptions}

\theoremstyle{plain}
\newtheorem*{thm*}{Theorem}

\theoremstyle{plain}

\input{symbols}

\definecolor{b2}{RGB}{51,153,255}
\definecolor{mygreen}{RGB}{80,180,0}
\definecolor{mint}{RGB}{62, 180, 137}

\title{Towards Anytime-Valid Statistical Watermarking}

\begin{document}

\ifdefined\isarxiv
\author{
Baihe Huang\thanks{
  \texttt{baihe\_huang@berkeley.edu}. University of California, Berkeley. }
  % examples of more authors
  \and
  Eric Xu\thanks{\texttt{erx@berkeley.edu}. University of California, Berkeley. }
  \and
  Kannan Ramchandran
  \thanks{
  \texttt{kannanr@berkeley.edu}. University of California, Berkeley. }
  \and
  Jiantao Jiao\thanks{
  \texttt{jiantao@eecs.berkeley.edu}. University of California, Berkeley. }
  \and
  Michael I. Jordan\thanks{
  \texttt{jordan@cs.berkeley.edu}. University of California, Berkeley. }
}
\date{}
\maketitle

\else
\twocolumn[
  \icmltitle{Towards Anytime-Valid Statistical Watermarking}



  \icmlsetsymbol{equal}{*}

  \begin{icmlauthorlist}
    \icmlauthor{Firstname1 Lastname1}{equal,yyy}
    \icmlauthor{Firstname2 Lastname2}{equal,yyy,comp}
    \icmlauthor{Firstname3 Lastname3}{comp}
    \icmlauthor{Firstname4 Lastname4}{sch}
    \icmlauthor{Firstname5 Lastname5}{yyy}
    \icmlauthor{Firstname6 Lastname6}{sch,yyy,comp}
    \icmlauthor{Firstname7 Lastname7}{comp}
    \icmlauthor{Firstname8 Lastname8}{sch}
    \icmlauthor{Firstname8 Lastname8}{yyy,comp}
  \end{icmlauthorlist}

  \icmlaffiliation{yyy}{Department of XXX, University of YYY, Location, Country}
  \icmlaffiliation{comp}{Company Name, Location, Country}
  \icmlaffiliation{sch}{School of ZZZ, Institute of WWW, Location, Country}

  \icmlcorrespondingauthor{Firstname1 Lastname1}{first1.last1@xxx.edu}
  \icmlcorrespondingauthor{Firstname2 Lastname2}{first2.last2@www.uk}

  \icmlkeywords{Machine Learning, ICML}

  \vskip 0.3in
]

\fi

% \maketitle
\setlength{\abovedisplayskip}{3.2pt}
\setlength{\belowdisplayskip}{3.2pt}
% \printAffiliationsAndNotice{}

\begin{abstract}
The proliferation of Large Language Models (LLMs) necessitates efficient mechanisms to distinguish machine-generated content from human text. While statistical watermarking has emerged as a promising solution, existing methods suffer from two critical limitations: the lack of a principled approach for selecting sampling distributions and the reliance on fixed-horizon hypothesis testing, which precludes valid early stopping. In this paper, we bridge this gap by developing the first e-value-based watermarking framework, \emph{Anchored E-Watermarking}, that unifies optimal sampling with anytime-valid inference. Unlike traditional approaches where optional stopping invalidates Type-I error guarantees, our framework enables valid, anytime-inference by constructing a test supermartingale for the detection process. By leveraging an anchor distribution to approximate the target model, we characterize the optimal e-value with respect to the \emph{worst-case log-growth rate} and derive the optimal expected stopping time. Our theoretical claims are substantiated by simulations and evaluations on established benchmarks, showing that our framework can significantly enhance sample efficiency, reducing the average token budget required for detection by 13-15\% relative to state-of-the-art baselines.
\end{abstract}

\input{intro}

\input{prelim}
\input{problem}
\input{theory}
\input{experiment}
\input{discussions}

\ifdefined\isarxiv
\bibliography{ref}
\bibliographystyle{icml2026}
\newpage
\appendix
\else

\bibliography{ref}
\bibliographystyle{icml2026}

%%%%%%%%%%%%%%%%%%%%%%%%%%%%%%%%%%%%%%%%%%%%%%%%%%%%%%%%%%%%%%%%%%%%%%%%%%%%%%%
%%%%%%%%%%%%%%%%%%%%%%%%%%%%%%%%%%%%%%%%%%%%%%%%%%%%%%%%%%%%%%%%%%%%%%%%%%%%%%%
% APPENDIX
%%%%%%%%%%%%%%%%%%%%%%%%%%%%%%%%%%%%%%%%%%%%%%%%%%%%%%%%%%%%%%%%%%%%%%%%%%%%%%%
%%%%%%%%%%%%%%%%%%%%%%%%%%%%%%%%%%%%%%%%%%%%%%%%%%%%%%%%%%%%%%%%%%%%%%%%%%%%%%%
\onecolumn

\fi

\newpage
\appendix

\input{e_proof}
\input{tau_proof}

\input{claims}

\input{add_experiments}

\end{document}

%% file: symbols.tex
\newcommand{\diag}{\mathrm{diag}}

\newcommand{\tr}{\mathrm{tr}}

\newcommand{\cP}{\mathcal{P}}

\newcommand{\ext}{\mathrm{ext}}

\newcommand{\Z}{\mathbb{Z}}
\newcommand{\R}{\mathbb{R}}
\newcommand{\E}{\mathbb{E}}

\newcommand{\cR}{\mathcal{R}}

\newcommand{\cA}{\mathcal{A}}
\newcommand{\cG}{\mathcal{G}}
\newcommand{\cE}{\mathcal{E}}
\newcommand{\cQ}{\mathcal{Q}}

\newcommand{\cS}{\mathcal{S}}
\newcommand{\cV}{\mathcal{V}}

\newcommand{\scomp}{\mathrm{SC}}

\newcommand{\flow}{\mathrm{flow}}

\newcommand{\probname}{\ifmmode\mathrm{ADW}\else ADW \fi}

%% file: intro.tex
\section{Introduction}

The revolutionary success of Large Language Models (LLMs) at generating human-like texts~\citep{brown2020language,bubeck2023sparks,chowdhery2023palm} has raised several societal concerns regarding the misuse of LLM outputs. Unregulated LLM outputs pose risks ranging from the contamination of future training corpora~\citep{shumailov2023curse,das2024under} to the propagation of disinformation~\citep{zellers2019defending,vincent2022ai} and academic misconduct~\citep{jarrah2023using,milano2023large}. Consequently, there is an urgent need for reliable detection mechanisms capable of distinguishing LLM-generated text from human-authored content.

To address this challenge, \citet{aaronson2022my} and \citet{kirchenbauer2023watermark} propose \emph{statistical watermarking} as a theoretically-grounded method to distinguish machine-generated content from a target distribution. These methods operate by injecting statistical signals into the generated texts in the decoding phase of LLMs. Mechanistically, the watermark generator couples the output tokens by a sequence of pseudorandom seeds. During detection, the detector reconstructs the seeds and performs a hypothesis test to check the dependence between the seeds and the observed tokens. This formulation allows the detection problem to be treated as a rigorous statistical hypothesis test~\citep{huang2023towards} where the detection power and sample complexity are governed by the randomness in the target distribution. 
Various designs for the underlying seed distributions have been explored, such as binary~\citep{christ2023undetectable}, exponential~\citep{kuditipudi2023robust}, and Gaussian~\citep{block2025gaussmark}.  Recently,  \citet{huang2025watermarking} advanced the concept of using an auxiliary open-source model as an \emph{anchor} to generate random
seeds and coupling them with tokens via speculative decoding ~\citep{leviathan2023fast}. This approach is motivated by the theoretical insight that watermarking achieves the optimal power when the seed distribution approximates the target distribution~\citep{huang2023towards}. By exploiting the distributional proximity between modern LLMs, anchor-based methods have shown promise in improving both detection efficiency
and robustness to removing-attacks~\citep{piet2023mark}. 
Empirically, watermarked texts can be detected within short token horizons ($<100$ tokens) and high success rate (nearly $50\%$ under paraphrasing) under various attacks.

Despite these successes, state-of-the-art watermarking remains constrained by two fundamental limitations. First, the design of generation and detection schemes lacks a unifying guiding principle.  \citet{huang2025watermarking} utilize a hashed green-red-list seed adapted from \citet{kirchenbauer2023watermark}, but this heuristic lacks rigorous justification. What is needed is a characterization of optimal generation and detection schemes under given anchor distributions.
Second, there exists a methodological disconnect between generation and detection. While text generation is inherently autoregressive and variable in length, current detection paradigms are \emph{fixed-horizon} and \emph{ad hoc}. In practice, for example, a detector may wish to process a stream of tokens and stop as soon as confidence is high. However, with the standard statistical paradigm based on p-values, it is not allowed to continuously monitor the test statistic and decide when to stop based on the current value---this is known as ``p-hacking'' and it invalidates the Type-I error guarantee. This inability to perform \emph{post-hoc} sequential analysis severely limits detection efficiency. These gaps highlight a critical research challenge:
\begin{center}
    \emph{How can we design provably efficient watermarking schemes leveraging anchor distributions that allow for valid early stopping?}
\end{center}

In this work, we provide the first formal resolution to this question. We formulate the problem of \emph{Anchored E-Watermarking}, where both the generator and detector share access to an anchor distribution $p_0$, aiming to watermark a family of distributions in the $\delta$-proximity of $p_0$. Departing from classical p-value analysis, we adopt \emph{e-values} as the central detection paradigm. E-values~\citep{vovk1993logic,shafer2011testmartingales,vovk2021values} are nonnegative random variables $E$ satisfying $\E_{\mathbf{H}_0}[E] \leq 1$ under the null. Unlike p-values, e-values arise from supermartingales, and it is therefore possible to preserve Type-I error guarantees under optional stopping based on ongoing analysis of data. We analyze the optimal e-value for watermark detection with respect to the worst-case log-growth rate~\citep{kelly1956new} and the expected stopping time~\citep{grunwald2020safe,waudby2025universal}, thus fully characterize the average per-step growth rate of evidence and sample efficiency. Our main theoretical contributions are summarized informally as follows:

\begin{theorem}[Informal version of \Cref{thm:log-growth} and \Cref{thm:stopping-time}]\label{thm:informal}
The optimal worst-case log-growth rate $\E_{\mathbf{H}_1}[\log E]$ under the alternative $\mathbf{H}_1$ is given by:
\begin{align*}
    J^* = h + \left(1-\frac{\delta}{2}\right)\log\left(1-\frac{\delta}{2}\right)
 + \frac{\delta}{2}\log\left(\frac{\delta}{2(n-1)}\right),
\end{align*}
where $h = H(p_0)$ is the Shannon entropy of the anchor distribution $p_0$ and $\delta>0$ is a robustness tolerance parameter.
Furthermore, the optimal expected stopping time to achieve a Type-I error $\alpha$ scales as $\frac{\log(1/\alpha)}{J^*}$.
\end{theorem}

To the best of our knowledge, this work represents the first application of e-values to the domain of statistical watermarking. By enabling valid sequential testing, our framework significantly improves detection efficiency, allowing the system to flag machine-generated text with fewer tokens than fixed-horizon counterparts. With the ability to early-stop, this sequential paradigm enhances robustness against adaptive attacks: because the detector can terminate immediately upon accumulating sufficient evidence, the watermark remains effective even if the attacker perturbs the text heavily in later segments (post-stopping). Consequently, our approach offers a theoretically rigorous and practically superior alternative to existing heuristic detection methods. Through experiments on real watermarking benchmark~\citep{piet2023mark}, we show that the theoretical scheme implied by \Cref{thm:informal} achieves consistently higher sample efficiency than state-of-the-art methods, reducing the token consumption by 13-15\% across various temperature settings.

\subsection{Related works}

\label{sec:related_work}

\paragraph{Statistical watermarking.}
Watermarking offers a white-box provenance mechanism for detecting LLM-generated text~\citep{tang2023science}, complementing post-hoc detectors and provenance tools developed for neural text generation~\citep{zellers2019defending}.
Classical digital watermarking and steganography provide a broad toolbox for embedding and extracting imperceptible signals under benign or adversarial channel edits~\citep{cox2007digital}. Early works in NLP literature studied watermarking and tracing of text via editing or synonym substitutions~\citep{venugopal-etal-2011-watermarking,rizzo2019fine,abdelnabi2021adversarial,yang2022tracing,kamaruddin2018review}.
In contrast, modern \emph{statistical} (a.k.a.\ generative) watermarking~\citep{aaronson2022my,kirchenbauer2023watermark} injects a secret, testable distributional bias into the sampling process, and detects this bias via hypothesis testing on the generated token sequence.
A rapidly growing theory studies efficiency and optimality of watermarking tests and encoders: results include finite-sample guarantees, information-theoretic limits, and constructions that are distortion-free or unbiased~\citep{huang2023towards,zhao2023provable,li2024statistical,block2025gaussmark,kuditipudi2023robust,hu2023unbiased,xie2025debiasing}.
Negative and hardness results highlight fundamental limitations against adaptive or distribution-matching adversaries~\citep{christ2023undetectable,christ2024pseudorandom,golowich2024edit}, motivating alternative design considerations such as distribution-preserving and public-key  schemes~\citep{wu2023distributionpreserving,liu2023unforgeable,fairoze2023publiclydetectable}.
Empirical robustness is commonly assessed under paraphrasing, editing and translation, with recent work studying cross-lingual failure modes and defenses~\citep{he2024translation}, and benchmarks/frameworks such as MarkMyWords and scalable pipelines for watermark evaluation and deployment~\citep{piet2023mark,zhang2024remarkllm,lau2024waterfall,dathathri2024scalable}.
Finally, a complementary line of work leverages semantic structure and auxiliary models to boost detection power under benign distributional structure, including semantic/paraphrastic watermarks and speculative-sampling-based schemes~\citep{ren2024robustsemanticsbasedwatermarklarge,liu2024adaptive,hou2024k,hou2024semstampsemanticwatermarkparaphrastic,fu2024watermarking,huang2025watermarking,leviathan2023fast}.
Our anchored e-value approach builds on these works by explicitly treating detection as anytime-valid sequential testing~\citep{chen2024online}, and by allowing model-assisted calibration under distribution shift~\citep{huang2025watermarking,he2024distributionadaptive,he2025distributionalembedding}.

\paragraph{E-values and sequential hypothesis testing.}
Sequential hypothesis testing studies inference procedures that remain valid under data-dependent stopping and streaming data~\citep{WaldSequentialAnalysis1947,Robbins1952SomeAO,breimualphanu2011optimal}.
Classical likelihood-ratio and $p$-value based methods can fail under optional stopping or composite hypotheses, motivating the use of nonnegative supermartingales / test martingales whose stopped values are valid evidence measures~\citep{shafer2011testmartingales,vovk2021values,Shafer2021Testing,grunwald2020safe,ramdas2023game}.
The resulting \emph{e-values} and \emph{e-processes} unify betting scores, likelihood ratios, and (stopped) Bayes factors. Additionally, e-values and e-processes support modular operations such as merging and calibration~\citep{vovk2021values,ramdasWang2025hypevalues}.
This framework has enabled principled notions of power and optimality, including log-/growth-rate criteria and sharp asymptotic rates for broad classes of betting-based tests~\citep{grunwald2020safe,waudby2025universal}.
E-values also interact fruitfully with multiple testing and online testing, yielding e-value analogues of the BH and wealth-based procedures and their extensions to structured settings~\citep{wangramdas2022fdrevalues,ramdas2017decayingmemory,ramdas2018saffron,ramdas2019dagger}.
Recent work further broadens the reach of e-values to new ML settings, such as conformal prediction and sequential monitoring of strategic systems~\citep{gauthier2025conformal_evalues,gauthier2026bettingequilibrium,aolaritei2025sgdcs}, and connects e-values to model-assisted efficiency gains via prediction-powered inference~\citep{wasserman2020universal,csillag2025predictionpoweredevalues}.
Our work draws on these foundations to construct e-values tailored to watermark detection which ensures validity under arbitrary stopping rules and enables sequential evidence accumulation in practical provenance audits.

%% file: prelim.tex
\section{Preliminaries}

\subsection{Statistical watermarking}

Statistical watermarking for Large Language Models (LLMs) provides a probabilistic framework for distinguishing machine-generated text from human-written text. Unlike post-hoc detection methods that rely on classifiers trained on model artifacts, watermarking actively embeds a statistical signal into the generation process. This signal is designed to be imperceptible to humans yet statistically significant to an algorithmic detector possessing a secret key.

\paragraph{Statistical watermarking as a hypothesis testing problem.}
Formally, statistical watermarking is cast as a hypothesis testing problem. Let $q$ denote the target distribution over the sample space $\mathcal{V}$, and let $\mathcal{S}$ represent the space of pseudorandom seeds. The watermarking protocol modifies the standard decoding mechanism of the LLM such that the output tokens $V \in \mathcal{V}$ and the pseudorandom seeds $S \in \mathcal{S}$ are sampled jointly from a watermarked distribution $\mathcal{P}_W$. 

The detection task seeks to determine if a given observed text $v$ was generated by the watermarked model: $\mathcal{D}: \mathcal{V} \times \mathcal{S} \to \{\text{Watermarked}, \text{Unwatermarked}\}$. This reduces to testing for independence between the tokens $v$ and the seeds $s$ reconstructed via the detector's key. We define the null and alternative hypotheses as follows~\citep{huang2023towards}:
\begin{itemize}
    \item \textbf{Null hypothesis $\textbf{H}_0$:} The text $v$ is generated independently of the seeds $s$ (e.g., by a human or an unwatermarked model).
    \item \textbf{Alternative hypothesis $\textbf{H}_1$:} The text $v$ and seeds $s$ are sampled from the joint watermarked distribution $\mathcal{P}_W$, implying they come from the watermarked model.
\end{itemize}

It is pertinent to distinguish this \textit{distributional} watermarking approach from classical instance-level watermarking techniques applied to static media, such as images or audio \cite{cox2007digital}. While classical methods embed a signature into a specific, fixed outcome (post-generation), statistical watermarking modifies the stochastic sampling process itself, ensuring that any realization from the model carries the statistical evidence of its origin without requiring rigid alterations to the final output.

\paragraph{Statistical guarantees.}
A robust watermarking scheme is characterized by its ability to provide three fundamental guarantees.

First, to preserve generation quality and ensure watermarked content remains indistinguishable from ordinary samples, the distance (e.g., KL-divergence) between the watermarked distribution $\mathcal{P}_W$ and the original target distribution $q$ must be minimized. Ideally, we seek a scheme that introduces zero distributional shift~\citep{hu2023unbiased}.

\begin{definition}[Distortion-free]\label{def:unbias}
    A watermark is considered \textit{distortion-free} (or unbiased) if the outcome marginal of the watermarked distribution $\mathcal{P}_W$ matches the target distribution $q$. Formally, a distortion-free watermark satisfies:
    \begin{align*}
        \sum_{s \in \mathcal{S}} \mathcal{P}_W(A, s) = q(A), \quad \forall A \subseteq \mathcal{V}.
    \end{align*}
\end{definition}

Second, it is often required in practice that the detector operates without access to the watermarked distribution since the underlying model's parameters are generally unknown to the detector. This means that the detection scheme needs to be designed for a \emph{family of target distributions} in a model-agnostic way~\citep{kuditipudi2023robust}, leading to the next concept.

\begin{definition}[Model-agnosticity]\label{def:agnostic}
    A watermark is \textit{model-agnostic} if the seed-marginal of $\mathcal{P}_W$ is chosen independently of the target distribution $q$. That is, the distribution of seeds $S$ does not depend on the model's logits.
\end{definition}

Third, the reliability of the watermark is quantified by standard statistical error rates.
The \emph{Type I error} (false positive rate) measures the probability of incorrectly identifying non-watermarked text (e.g., human-written) as watermarked. Given the high stakes of false accusations (e.g., in academic integrity), this error must be bounded by a small significance level $\alpha$:
\begin{align*}
    \mathbb{P}_{H_0}(\mathcal{D}(V, S) = \text{Watermarked}) \leq \alpha.
\end{align*}
The \emph{Type II error} (false negative rate) measures the probability that watermarked text fails to be detected. Minimizing this error is equivalent to maximizing the statistical power of the test. For a target error rate $\beta$, we require:
\begin{align*}
    \mathbb{P}_{H_1}(\mathcal{D}(V, S) = \text{Unwatermarked}) \leq \beta.
\end{align*}

\subsection{Sequential hypothesis testing and e-values}

Consider a measurable space $(\Omega, \mathcal{F})$ and a family of probability distributions $\mathcal{P}$. We are interested in testing a null hypothesis $\mathcal{H}_0: P \in \mathcal{P}_0 \subset \mathcal{P}$ against an alternative $\mathcal{H}_1: P \in \mathcal{P}_1 = \mathcal{P} \setminus \mathcal{P}_0$. 
An e-value is a random variable whose expectation is bounded by unity under the null hypothesis~\citep{vovk2021values}.

\begin{definition}[E-value]
A nonnegative random variable $E$ is an e-value for $\mathcal{H}_0$ if for all $P \in \mathcal{P}_0$,
\begin{align*}
    \mathbb{E}_P[E] \le 1.
\end{align*}
\end{definition}

Intuitively, an e-value represents the amount of wealth a gambler would have after betting against the null hypothesis, in a game where the game is fair or unfavorable under $\mathcal{H}_0$, and starting with an initial wealth of one. On the other hand, if $E$ becomes large, this suggests that the null hypothesis is unlikely to be true.

\paragraph{Sequential evidence and stopping time guarantees.}

The primary advantage of e-values lies in their behavior regarding stopping times. In a sequential setting, we are given a filtration $(\mathcal{F}_t)_{t \ge 0}$ and construct a sequence of e-values $(E_t)_{t \ge 0}$. If $(E_t)_{t \ge 0}$ is a nonnegative supermartingale with respect to the null distributions (often called a \textit{test martingale}), it allows for \textit{anytime-valid} inference.
This property derives from Ville's inequality, a time-uniform generalization of Markov's inequality.

\begin{theorem}[Ville's inequality]
Let $(E_t)_{t \ge 0}$ be a nonnegative supermartingale with respect to $P \in \mathcal{P}_0$ such that $E_0 \le 1$. Then for any $\alpha \in (0,1)$,
\begin{align*}
    P\left( \exists t \ge 0 : E_t \ge \frac{1}{\alpha} \right) \le \alpha.
\end{align*}
\end{theorem}

This result guarantees that a researcher can track the e-value process continuously and stop at any data-dependent time $\tau$ (a stopping time) while maintaining Type-I error control. Specifically, $\mathbb{E}_P[E_\tau] \le 1$ holds for any stopping time $\tau$ (possibly unbounded), a kind of guarantee which is not available for p-values.

%% file: problem.tex
\section{Problem Formulation}

In this section, we formulate the problem of \emph{Anchored E-Watermarking}. Similar to statistical watermarking, the goal is to embed statistical signals into samples from a target distribution to enable detection. However, in Anchored E-Watermarking, the generator and the detector share access to an anchor distribution that serves as a robust \emph{a priori} estimate of the target distribution. This setting is common in practice; for example, when watermarking Large Language Models or image generators, the target distribution is known to be human language or natural images, respectively. Consequently, watermarking efficacy can be improved by leveraging this structure. Besides, Anchored E-Watermarking uses e-values for detection, thus enjoying improved efficiency in sequential testing. In the following, we introduce the specific components of Anchored E-Watermarking.

\paragraph{Anchor distribution.}

In the Anchored E-Watermarking framework, we assume that both the generator and the detector share access to an \emph{anchor distribution} $p_0 \in \Delta(\cV)$, which is in the same probability simplex as the target distribution $q$. This $p_0$ serves as the best \emph{a priori} estimate of the target distribution. For instance, in the context of watermarking Large Language Models (LLMs), $p_0$ could be an open-source, smaller-scale LLM such as Qwen3-8B~\citep{qwen3technicalreport}. Leveraging its role as a known reference, we designate $p_0$ as the marginal distribution for the watermark signal (or random seed), denoted as $s$. Therefore, the seed space $\cS$ is equal to $\cV$ in our framework. Throughout the paper, we assume $p_0$ satisfies the condition $\inf_{v \in \cV}p_0(v) > \delta$ for a positive robustness tolerance parameter $\delta$. This condition ensures $p_0$ has enough randomness as required for watermarking~\citep{aaronson2022my} and that the $\delta$-neighborhood defined below is a strict subset of the probability simplex.

\paragraph{Target distribution.}

The target distribution $q \in \Delta(\cV)$ represents the desired distribution from which (watermarked) output samples $v$ are generated. Consistent with the definition of the anchor, we premise that $q$ remains sufficiently proximal to $p_0$. Formally, we assume $q$ resides within the $\delta$-neighborhood of the anchor:
\begin{align*}
    \cQ(p_0,\delta) = \left\{q \in \Delta(\cV): \|q - p_0\|_1 \leq \delta\right\}.
\end{align*}
Here $\|\cdot\|_1$ is the $\ell_1$ distance and the radius $\delta>0$ represents a robustness tolerance parameter that controls how the true target distribution $q$ can deviate from the anchor. 
This assumption holds for most practical statistical watermarking scenarios; for example, in the LLM domain, high-performing models tend to converge towards the same underlying distribution of natural language, resulting in low statistical divergence between them.

\paragraph{Generator.}

The objective of the generator is to sample an output $v$ from the target distribution $q$ while embedding a statistical signal $s$. To achieve this, the generator constructs a joint distribution (or coupling) $w$ over $v$ and $s$, subject to the constraint that the marginals recover the target and anchor distributions, respectively. The feasible set of valid couplings is defined as:
\begin{align*}
    \cP(p_0,q) = \bigg\{w \in \Delta(\cV \times \cV): \sum_{s \in \cV}w(v,s) = q(v), \quad \sum_{v\in \cV}w(v,s) = p_0(s)\bigg\}.
\end{align*}
Note that the generator has access to $q$, as is common in practice (e.g., when the generator is a service provider). During generation, the system samples the outcome $v$ and the signal $s$ jointly from $w$. This procedure ensures that the marginal distribution of the outcome $v$ is unbiased (satisfying \Cref{def:unbias}), while the watermark signal is embedded within the statistical dependency between $v$ and $s$.

\paragraph{E-value.}

In the detection phase, the detector receives an outcome-signal pair $(v,s)$ and computes an e-value $e(v,s) \geq 0$. The objective is to distinguish between the following hypotheses:
\begin{align*}
    \textbf{H}_0: &~ v \text{ and } s \text{ are sampled independently},\\
    \textbf{H}_1: &~ (v,s) \text{ are sampled from the joint coupling } w.
\end{align*}
The detector rejects the null hypothesis if $e(v,s) > 1/\alpha$. To guarantee a Type-I error rate of at most $\alpha$, the expectation of the e-value under the null hypothesis must be bounded by $1$. Adhering to the model-agnostic constraint (\Cref{def:agnostic}), the scoring function $e(\cdot, \cdot)$ must be defined prior to observing the specific target distribution $q$. The design of the e-value relies solely on the constraint that the target distribution lies within the anchor's neighborhood, i.e., $q \in \cQ(p_0,\delta)$. Consequently, to ensure validity, the e-value must satisfy the expectation bound uniformly across the entire uncertainty set $\cQ$:
\begin{align}\label{eq:e-null}
    \sup_{q \in \cQ(p_0,\delta)} \mathbb{E}_{v \sim q, s \sim p_0} [e(v,s)] \leq 1.
\end{align}
We let $\cE$ denote the set of all nonnegative functions $e: \cV \times \cV \to \R$ satisfying Eq.~\eqref{eq:e-null}.

\subsection{Robust log-optimality}

While the definition of an e-value ensures validity (safety) under the null, it does not guarantee power (growth) under the alternative. To reject $\textbf{H}_0$ effectively, we desire the e-value to be large when the data is generated from the alternative distribution $\textbf{H}_1$. The standard criterion for selecting an e-value is \textit{log-optimality}, often referred to as the ``Kelly criterion'' in the finance literature~\citep{kelly1956new}. This approach seeks to maximize the expected logarithmic growth rate of evidence against the null, under the alternative. For simple hypotheses where $\mathcal{P}_0 = \{P\}$ and $\mathcal{P}_1 = \{Q\}$, the likelihood ratio $e^* = \frac{dQ}{dP}$ is log-optimal, such that
\begin{align*}
    \mathbb{E}_Q[\log e^*] \ge \mathbb{E}_Q[\log e]
\end{align*}
for any valid e-value $e$.

In Anchored E-Watermarking, the alternative is defined by the joint coupling $w$ constructed by the generator. However, since the target distribution $q$ lies in a set $\cQ(p_0,\delta)$ that is unknown to the detector (i.e., the designer of the e-value), we must optimize the worst-case log-growth rate over arbitrary $q \in \cQ(p_0,\delta)$. This gives rise to the following robust log-optimality problem.

\begin{problem}[Robust log-optimality]\label{prob:robust-log-optimal}
An e-value $e$ is \emph{robust log-optimal} in Anchored E-Watermarking if it optimizes the following objective:
\begin{align}\label{eq:log-growth}
    \sup_{e \in \cE} \inf_{q \in \cQ(p_0,\delta)} \sup_{w \in \cP(p_0,q)} &~  \sum_{v,s \in \cV} w(v,s) \cdot \log e(v,s).
\end{align}
\end{problem}

In this formulation, the objective $\sum w(v,s) \log e(v,s)$ represents the log-growth rate under the alternative. The inner maximization $\sup_{w \in \cP(p_0,q)}$ reflects the generator's optimization of the coupling $w$ given knowledge of $q$; the minimization $\inf_{q \in \cQ(p_0,\delta)}$ enforces robustness against the worst-case target distribution in the family; and the outer maximization $\sup_{e \in \cE}$ seeks the optimal detector without access to $q$ subject to model-agnosticy (\Cref{def:agnostic}).

\begin{remark}[Relationship to growth rate optimality in the worst case (GROW)~\citep{grunwald2020safe}]
GROW studies the worst-case optimal expected capital growth rate under composite null $\mathcal{P}_0$ and alternative $\mathcal{P}_1$:
\begin{align*}
    \mathrm{GROW}(\mathcal{P}_1) = \sup_{E\in \mathcal{E}(\mathcal{P}_0)} \inf_{\theta \in \mathcal{P}_1}\E_{P_\theta}[log E],
\end{align*}
where $ \mathcal{E}(\mathcal{P}_0)$ is the set of all valid e-values for the null $\mathcal{P}_0$.
Therefore, our problem in Eq.~\eqref{eq:log-growth} can be seen as generalizing GROW into an `active' scenario where a generator (corresponding to $\sup_{w \in \cP(p_0,q)}$) seeks to maximizes the power after the worst-case hypothesis selection $\inf_{\theta \in \mathcal{P}_1}$. Note that by designing the coupling $w$, the generator essentially alters the alternative, so our problem can not be reduced to GROW even if the watermark generator is fixed.
\end{remark}

\subsection{Stopping time}

While \Cref{prob:robust-log-optimal} addresses the one-step growth rate, it remains to analyze the sample complexity of the framework. In the context of sequential hypothesis testing, sample efficiency is characterized by the stopping time $\tau_\alpha = \inf\left\{n \in \Z_+: W_n \geq \log(1/\alpha)\right\}$ under the alternative, where $$W_n = \sum_{t=1}^{n}\log e(v^t,s^t),$$ is the accumulated wealth process. This quantity represents the number of samples required to reject the null hypothesis. We consider a stochastic process $\mu(\cA,\cG)$ governed by the interaction between an adversary $\cA$ and a generator $\cG$. At each step $t \in \Z_+$:
\begin{itemize}
    \item The adversary $\cA$ selects $q^t \in \cQ(p_0,\delta)$ based on the e-value $e$ and the history $(v^1,s^1),\dots,(v^{t-1},s^{t-1})$;
    \item The generator $\cG$ specifies the joint coupling $w^t$ given $q^t$ and $p_0$, and draws the sample $(v^t,s^t) \sim w^t$.
\end{itemize}
Due to the adaptive nature of the adversary, the outcome sequence $v^1, v^2, \dots$ may be generated autoregressively, conditional on previous outcomes. This formulation extends the i.i.d.\ setting of \citet{huang2023towards} to arbitrary dependent distributions. In this setting, we write the stopping time as $\tau_\alpha(e)$ to highlight the dependence on the e-value $e$. In the context of sequential testing \citep{breimualphanu2011optimal,chugg2023auditing,shekhar2023nonparametric,kaufmann2021mixture,waudby2025universal}, the expected stopping time $\mathbb{E}[\tau_\alpha]$ typically scales linearly with $\log(1/\alpha)$, modulated by an information-theoretic quantity capturing the complexity of the testing problem. Therefore, we are explicitly interested in the following question:

\begin{problem}[Sample efficiency]\label{prob:stopping-time}
For any e-value $e$ satisfying the constraint in Eq.~\eqref{eq:e-null}, define its \emph{sample complexity} by 
\begin{align*}
    \scomp(e) = \sup_{\cA} \inf_{\cG} \liminf_{\alpha \to 0} \frac{\mathbb{E}_{\mu(\cA,\cG)}[\tau_\alpha(e)]}{\log(1/\alpha)}.
\end{align*}
We are interested in identifying the e-value that minimizes this sample complexity.
\end{problem}

Note that this formulation is stronger than the classical expected stopping time because we allow the adversary to alter the distribution at each time step. This implies that samples are not generated in an i.i.d.\ fashion, a distinction crucial for applications such as the watermarking of autoregressive language models.

%% file: theory.tex
\section{Theoretical Results}

In this section, we answer the questions posed in the previous section. We first derive the optimal detector for the single-step decision problem and then extend this analysis to the sequential setting to characterize the fundamental limit of sample efficiency. Let $n$ denote the cardinality of the sample space $\cV$ and $\mathbf{e}_v$ denote the one-hot vector (i.e., Dirac measure) that assigns mass $1$ to $v \in \cV$.

\subsection{Optimal log-growth rate}

We begin by solving the robust log-optimality problem defined in \Cref{prob:robust-log-optimal}. The following theorem provides a closed-form solution for both the optimal e-value and the worst-case-optimal generator behavior.

\begin{theorem}[Log-growth rate]\label{thm:log-growth}
The optimal e-value that solves the objective in Eq.~\eqref{eq:log-growth} is given by:
\begin{align}\label{eq:optimal-e}
e^*(v,s)
=
\begin{cases}
\frac{1-\delta/2}{p_0(s)}, & s=v,\\
\frac{\delta}{2(n-1)p_0(s)}, & s\neq v,
\end{cases}
\end{align}
where $n = |\cV|$. Furthermore, for any target distribution $q \in \cQ(p_0,\delta)$ decomposed as $q = \sum_{i=1}^k \lambda_i q_i$ where $q_i = p_0 + \frac{\delta}{2}\cdot (\mathbf{e}_{v_i} - \mathbf{e}_{s_i}), v_i \neq s_i \in \cV$ are the extreme points of $\cQ(p_0,\delta)$, 
the optimizer of the inner maximization problem (the generator's optimal coupling) is:
\begin{align}\label{eq:coupling}
    w^* = \sum_{i=1}^k \lambda_i \cdot \left(\frac{\delta}{2} \cdot (\mathbf{e}_{v_i}\mathbf{e}_{s_i}^\top - \mathbf{e}_{s_i}\mathbf{e}_{s_i}^\top) + \diag(p_0)\right).
\end{align}
The optimal robust log-growth rate is equal to:
\begin{align}\label{eq:optimal-log-growth}
    J^* = h + \left(1-\frac{\delta}{2}\right)\log\left(1-\frac{\delta}{2}\right)
 + \frac{\delta}{2}\log\left(\frac{\delta}{2(n-1)}\right),
\end{align}
where the term $h = H(p_0)$ is the Shannon entropy of the anchor distribution $p_0$.
\end{theorem}

The result in \Cref{eq:optimal-log-growth} fundamentally express the optimal growth rate of e-values in Anchored E-Watermarking. The first term is the entropy of the anchor distribution, which mathematically formalizes the intuition that watermarking is inherently easier for distributions with high entropy. 
The second term can be written as $-H(\nu_\delta)$, the negative entropy of the categorical distribution $\nu_\delta = (1-\delta/2, \delta/(2(n-1)),\dots,\delta/(2(n-1))) \in \Delta([n])$, which is a decreasing function of $\delta$. This explicitly characterizes the power-robustness tradeoff determined by the tolerance parameter $\delta$. 
Written as $H(p_0) - H(\nu_\delta)$, the optimal log-growth rate is the achieved information rate of an Additive Noise Channel $Y = X \oplus Z$ where $X \sim p_0, Z \sim \nu_\delta$. 

\begin{remark}[Relationship to SEAL~\citep{huang2025watermarking}]
SEAL proposes to use a smaller language model as $p_0$ and speculative decoding as the generator.
Note that the optimizer of the inner problem given in Eq.~\eqref{eq:coupling} is exactly the maximal coupling given by speculative decoding. Therefore, \Cref{thm:log-growth} confirms that the choice of generator in \citet{huang2025watermarking} is optimal. However, the optimal detector given by Eq.~\eqref{eq:optimal-e} is different from the detection rule in SEAL: this gap is corroborated by experiments in Section~\ref{sec:exp-real}.
\end{remark}

\subsection{Sample efficiency}

Having characterized the one-step optimal growth rate, we now analyze the long-term performance of the watermark in a sequential setting. The following theorem connects the log-growth rate $J^*$ to the sample complexity required to reject the null hypothesis against an adaptive adversary.

\begin{theorem}[Expected stopping time]
\label{thm:stopping-time}
Let $\mu(\cA,\cG)$ be the stochastic process defined in \Cref{prob:stopping-time} and $\tau_\alpha(e)$ be the stopping time under e-value $e$. Let $J^*$ be the optimal rate defined in Eq.~\eqref{eq:optimal-log-growth}. We have the following bounds on sample efficiency:
\begin{itemize}
    \item \textbf{Lower bound (converse):} For any valid e-value $e \in \cE$:
\begin{align*}
    \sup_{\cA} \inf_{\cG} \liminf_{\alpha \to 0} \frac{\mathbb{E}_{\mu(\cA,\cG)}[\tau_\alpha(e)]}{\log(1/\alpha)} \geq \frac{1}{J^*}.
\end{align*}
\item \textbf{Upper bound (achievability):} For the optimal e-value $e^*$ defined in \Cref{thm:log-growth}:
\begin{align*}
    \sup_{\cA} \inf_{\cG} \liminf_{\alpha \to 0} \frac{\mathbb{E}_{\mu(\cA,\cG)}[\tau_\alpha(e^*)]}{\log(1/\alpha)} = \frac{1}{J^*}.
\end{align*}
\end{itemize}
\end{theorem}

\Cref{thm:stopping-time} establishes $1/J^*$ as the fundamental information-theoretic limit of sample complexity for Anchored E-Watermarking. It demonstrates that the e-value $e^*$ derived from the one-step greedy optimization is not only locally optimal but also globally optimal for sequential testing. Crucially, this optimality holds even against an adaptive adversary that can vary the target distribution at every step, provided the distribution remains within the $\delta$-neighborhood of the anchor. 

\begin{remark}
[Relationship with the rates in \citet{huang2023towards}]
\citet{huang2023towards} shows that the minimum number of samples required to watermark with Type I error $\alpha$ scales as $\log(1/\alpha)\cdot\frac{\log(1/h)}{h}$, where $h$ is the average entropy per token. While our rate $\frac{\log(1/\alpha)}{h}$ has the same scaling in terms of $\alpha$, its dependence on the entropy $h$ is improved. This is because \citet{huang2023towards} study an asymptotic regime where $h \to 0$, while we fix an anchor distribution with lower bounded entropy arising from the condition $\inf_{v \in \cV}p(v) > \delta$.
\end{remark}

%% file: experiment.tex
\section{Experiments}

\subsection{Synthetic experiments}

\begin{figure*}[t]
\hfill
\begin{subfigure}[b]{.325\linewidth}
\includegraphics[width=\linewidth]{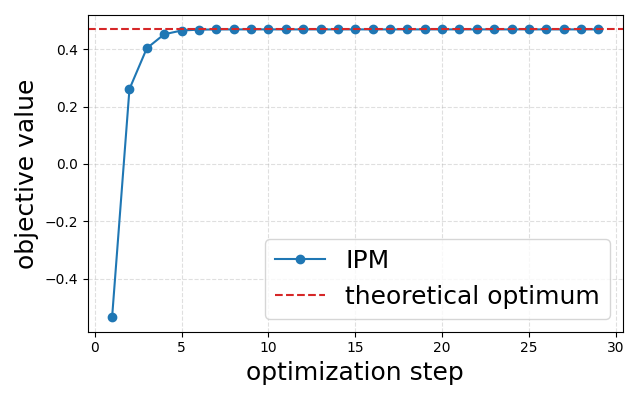}
\subcaption{$p = 0.2$}
\end{subfigure}
\hfill
\begin{subfigure}[b]{.325\linewidth}
\includegraphics[width=\linewidth]{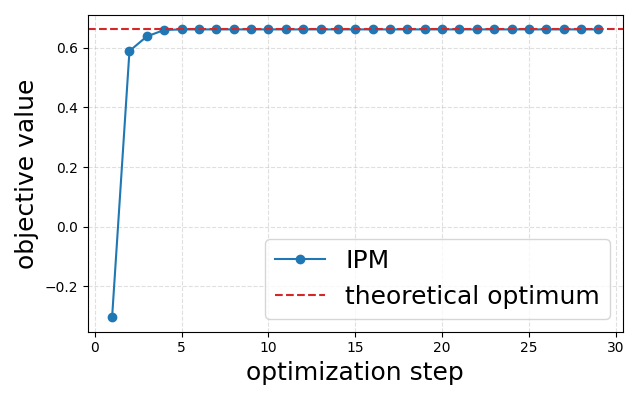}
\subcaption{$p = 0.5$}
\end{subfigure}
\hfill
\begin{subfigure}[b]{.325\linewidth}
\includegraphics[width=\linewidth]{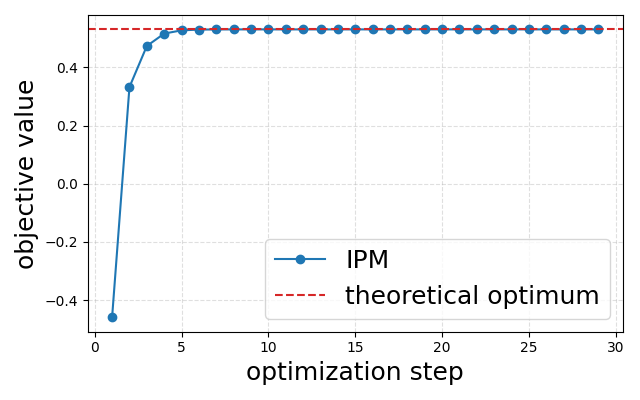}
\subcaption{$p = 0.75$}
\end{subfigure}
\caption{Simulation of the two-token case for the log growth problem in Eq.~\eqref{eq:log-growth}. Three separate anchor distributions are used each with parameter $\delta = 0.01$. We solve the simplified maxmin problem using the CLARABEL interior point method solver which is run for $30$ steps. The theoretical optimum is computed as in Eq.~\eqref{eq:optimal-log-growth}.}
\label{fig:log-growth}
\end{figure*}

In this section, we verify our theoretical results with synthetic experiments on a family of simple target distributions.
Here, we consider the two-token case where $\lvert \cV\rvert = 2$ and for simplicity of notation, we let $\mathcal{V} = \{0,1\}$. Any anchor distribution we consider belongs to a one-parameter Bernoulli family $p_0 = \begin{bmatrix}
    p & 1-p
\end{bmatrix}$ for $p \in (0,1)$.
\paragraph{Log-growth optimality.}
We solve the optimization problem in Eq.~\eqref{eq:log-growth} numerically and compare the objective curve with the theoretical prediction.
With $n=2$, the problem can be simplified to the following maxmin problem:
$$\sup_e\sum_{v} p_0(v)\log e(v,v) + \frac{\delta}{2}\min\left\{\log \frac{e(0,1)}{e(1,1)}, \log \frac{e(1,0)}{e(0,0)}\right\}$$
subject to the constraint in Eq.~\eqref{eq:e-null} with $\cQ(p_0,\delta)$ replaced with the set of vertices $\cQ_{\text{ext}}(p_0,\delta) := \{p_0 \pm \frac{\delta}{2}(\mathbf{e}_0-\mathbf{e}_1)\}$. Using the CLARABEL interior point method solver \cite{Clarabel_2024}, we obtain the results in Fig. \ref{fig:log-growth}. We choose $\delta = 0.01$ and $p = 0.2,0.5,0.75$ corresponding to three separate anchor distributions $p_0$. For each $p_0$, we cold start the IPM solver and run it for $30$ steps with maximum allowed step size $0.99$. We observe that our numerical solution to Eq.~\eqref{eq:log-growth} converges to the proposed theoretical value in Eq.~\eqref{eq:optimal-log-growth}, hence verifying our theoretical results in the two-token case. 

\paragraph{Stopping time.}
In this setting, we simulate sequential testing with the optimal e-value $e^*$, as in Eq.~\eqref{eq:optimal-e}. Following Theorem \ref{thm:log-growth}, we let $\mathcal{A}^*$ be the adversary that chooses $q_t = q^*$ for all $t$ and $\cG^*$ be the generator that selects $w_t = w^*$ in response to $\mathcal{A}^*$ for all $t$. Under this setting, we simulate the stopping time over several $\alpha$ values to obtain the results in Fig. \ref{fig:stopping-time}. Here, we choose the parameters $\delta = 0.1$ and $p= 0.2,0.5,0.75$. For each anchor $p_0$, we compute the corresponding $q^*$, $w^*$, and $e^*$ using the closed-form equations in Theorem \ref{thm:log-growth}. We select $30$ values of $\alpha$ evenly log-spaced from $10^{-2}$ to $10^{-120}$. For each $\alpha$, we compute $10000$ stopping-times by generating sequences $(V^t,S^t)_{t \geq 1} \sim w^*$ and computing the resulting $\log e^*(V^t,S^t)$. We then average the stopping times to form an estimate of $\mathbb{E}_{\mu(\mathcal{A}^*,\cG^*)}[\tau_{\alpha}(e^*)]$. In Fig. \ref{fig:stopping-time}, we plot our estimates of $\mathbb{E}_{\mu(\mathcal{A}^*,\cG^*)[\tau_{\alpha}(e^*)]}/(\log(1/\alpha))$ which show that as $\alpha \downarrow 0$, the estimates converge to the theoretical rate of $1/J^*$, thereby validating the result in Theorem \ref{thm:stopping-time}.

\begin{figure*}[t]
\hfill
\begin{subfigure}[h]{.325\linewidth}
\includegraphics[width=\linewidth]{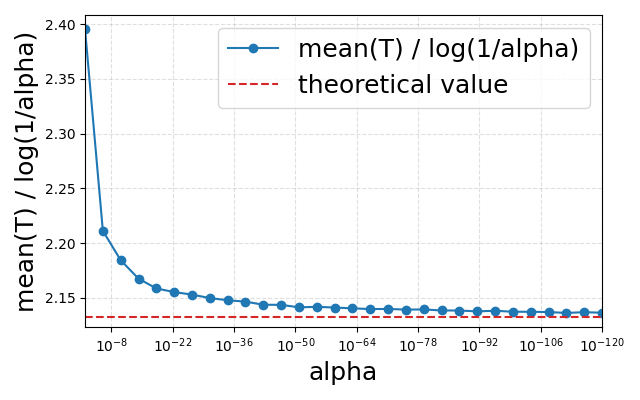}
\subcaption{$p = 0.2$}
\end{subfigure}
\hfill
\begin{subfigure}[h]{.325\linewidth}
\includegraphics[width=\linewidth]{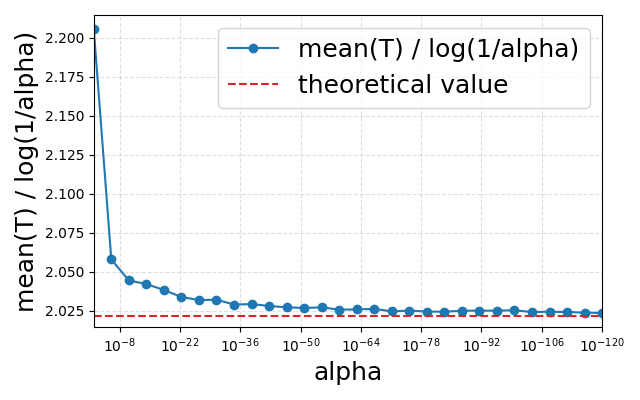}
\subcaption{$p = 0.5$}
\end{subfigure}
\hfill
\begin{subfigure}[h]{.325\linewidth}
\includegraphics[width=\linewidth]{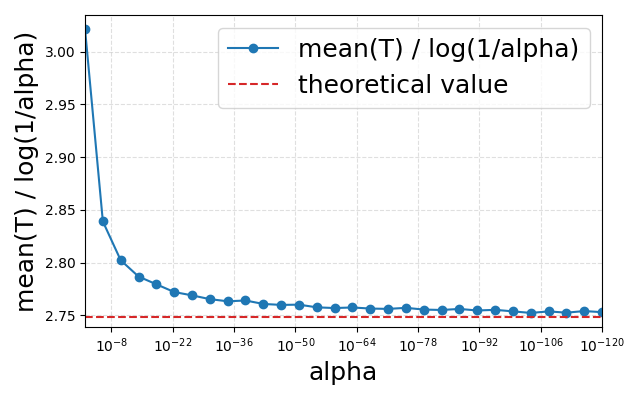}
\subcaption{$p = 0.75$}
\end{subfigure}
\caption{Simulation of the two-token case for the stopping time problem $\text{SC}(e^*)$. We simulate for three different anchor distributions $p_0 = \begin{bmatrix}
    p & 1-p
\end{bmatrix}$ with $\delta = 0.1$ and estimate average stopping times $\mathbb{E}[\tau_{\alpha}]$ for $\alpha$ values ranging from $10^{-2}$ to $10^{-120}$. Each $\mathbb{E}[\tau_{\alpha}]$ is estimated by simulating $10000$ stopping times $\tau_\alpha$. The plots above display graphs of $\frac{\mathbb{E}[\tau_\alpha]}{\log(1/\alpha)}$ with red dashed lines equal to $1/J^*$ where $J^*$ is as in Eq.~\eqref{eq:optimal-log-growth}. We observe that convergence to the theoretical optimum is obtained for sufficiently small $\alpha$.}
\label{fig:stopping-time}
\end{figure*}

\subsection{Experiments on real data}\label{sec:exp-real}

We evaluate the performance of our watermarking scheme in \Cref{thm:log-growth} by comparing it with several baselines on a real watermarking task.

\paragraph{Setup.}
We select Llama2-7B-chat~\citep{touvron2023llama} with temperature $0.7$ as target distribution (i.e., the model to be watermarked) and Phi-3-mini-128k-instruct~\citep{abdin2024phi3technicalreporthighly} as the anchor distribution. We evaluate our watermark using the \textsc{MarkMyWords} benchmark~\citep{piet2023mark}, which is an open-source benchmark designed to evaluate symmetric key watermarking schemes. 

\textsc{MarkMyWords} generates 300 outputs spanning three tasks---book summarization, creative writing, and news article generation---which mimic potential misuse scenarios. 
Then, watermarked outputs undergo a set of transformations mimicking realistic user perturbation or attacks: (1) character-level perturbations (contractions, expansions, misspellings and typos); (2) word-level perturbations (random removal, addition, and swap of words in each sentence, replacing words with synonyms); (3) text-level perturbations (paraphrasing, translating the output to another language and back).
Detection is run in a sequential test setting with early stopping. For baseline methods, we apply Bonferroni corrections to preserve \emph{anytime valid} Type-I error guarantees: let $p_k$ denote the p-value at time step $k \in \Z_+$, reject at the first time ${p_k} < \frac{\alpha}{k(k+1)}$.

\paragraph{Results.}

In this section, we present a comprehensive comparison of our method against these baselines in terms of quality and size. 
Quality measures the utility of the watermarked text. It is computed using Llama-3~\citep{dubey2024llama} with greedy decoding as a judge model, and ranges from zero to one.
{Size} represents the median number of tokens required to detect the watermark at a given p-value, computed over a range of perturbations listed above. All experiments enforce a Type-I error constraint of $\alpha = 0.02$. Lower values indicate higher efficiency.
We compare against five state-of-the-art watermarking schemes: 
``Distribution Shift''~\citep{kirchenbauer2023watermark}, 
``Exponential''~\citep{aaronson2022watermarking}, 
``Binary''~\citep{christ2023undetectable},
``Inverse Transform''~\citep{kuditipudi2023robust}, and ``SEAL''~\citep{huang2025watermarking}.

\begin{table}[h]
    \centering
    \begin{NiceTabular}{l|cc}
        \toprule
        {\bf Scheme
        } & {\bf Quality} ($\uparrow$) & {\bf Size} ($\downarrow$) \\
        \hline
        Exponential & 0.907	& $\infty$ \\
        Inverse Transform & 0.917 &	734.0 \\
        Binary & {\bf 0.919} &	$\infty$ \\
        Distribution Shift   & 0.912 &	145.0 \\
        SEAL   & 0.901 &	84.5 \\
        {\bf \Cref{thm:log-growth}} & {\bf 0.919} & {\bf 72.0} \\ 
        \bottomrule
    \end{NiceTabular}
    \vspace{.5em}
    \caption{{\fontfamily{ptm}\selectfont Comparison of our e-value-based watermarking scheme in \Cref{thm:log-growth} with baselines across quality and size. We report the median over different private keys and perturbation methods. The best result in each category is highlighted in bold. $\infty$ size suggests over half of the watermarked generations fail to be detected after perturbation. E-value-based watermarking demonstrates higher efficiency comparing with all baselines.}}
    \label{tab:comparison}
\end{table}

As shown in Table~\ref{tab:comparison}, our Anchored E-Watermarking framework significantly outperforms baseline methods in terms of detection efficiency while maintaining high generation quality. Specifically, our method improves over SEAL, confirming the superiority of the optimal e-value. Furthermore, we achieves nearly $2 \times$ speed improvement (from $145$ to $72.0$ tokens) compared to the best non-anchored baseline. Crucially, this efficiency gain does not come at the cost of text quality: the quality score of our method remains competitive with the baselines, demonstrating that E-value-based watermarking offers a superior trade-off between detectability and utility.
Due to space constraints, we defer additional experiment results to Appendix~\ref{sec:add_exp}.

%% file: discussions.tex
\section{Discussions}

We have introduced Anchored E-Watermarking, a novel framework that bridges the gap between optimal sampling and anytime-valid inference in statistical watermarking. By shifting the detection paradigm from p-values to e-values, we addressed the critical limitation of fixed-horizon testing, enabling valid optional stopping without compromising Type-I error guarantees.

Moreover, we characterized the optimal e-value with respect to the worst-case log-growth rate and derived the optimal expected stopping time, providing a rigorous foundation for watermarking in the presence of an anchor distribution. Empirically, our results on real-world language models demonstrate that this principled approach translates into substantial gains in efficiency. Our method identifies watermarked content with significantly fewer tokens than state-of-the-art heuristics while preserving generation quality.

As the first application of e-values to statistical watermarking, this framework opens new avenues for efficient detection mechanisms, with future works including extension to more flexible anchor distributions or investigating the game-theoretic implications of e-watermarking against incentivized adversaries.

%% file: e_proof.tex
\paragraph{Notation. } Let $[n]$ be a the set $\{1,\dots,n\}$ and $\Delta([n])$ denote the probability simplex over $[n]$. For the simplicity of notations, we let $\cV = [n]$. Let $S_n$ denote the permutation group of $\{1,\dots,n\}$. For any $\sigma \in S_n$, $\sigma^i$ means applying permutation $\sigma$ $i$-times. For any matrix $M \in \mathbb{R}^{n\times n}$, let $M(x,y), M(x,:), M(:,y)$ denote the entry on $x$-th row and $y$-th column, the $x$-th row, and the $y$-th column respectively. 

\section{Proof of \Cref{thm:log-growth}}
\label{sec:e-proof}
\begin{proof}
Let 
\begin{align*}
    \cR = \left\{r \in \R^{n \times n}: \sum_{s \in \cV} r(v,s) = 1, ~r(v,s) \geq 0,~ \forall v,s \in \cV \right\}.
\end{align*}
By \Cref{lem:normalizationconstant}, the original problem is equivalent to
\begin{align*}
    J': ~\sup_{r \in \cR} \inf_{q \in \cQ(p_0,\delta)} \sup_{w \in \cP(p_0,q)} \sum_{v,s \in \cV} w(v,s) \cdot \left(\log r(v,s) - \log p_0(s)\right),
\end{align*}
with $r(v,s) = {p_0(s)e(v,s)} \in \cR$.
Note that 
\begin{align*}
    \sum_{v,s \in \cV} w(v,s) \cdot (- \log p_0(s)) = \sum_{s \in \cV} p_0(s) \cdot (- \log p_0(s)) = H(p_0)
\end{align*}
where $H(p_0)$ is the entropy of $p_0$. 
By \Cref{prop:reform-problem}, the optimum value of $J'$ is equal to 
\begin{align*}
(1-\tfrac{\delta}{2})\log\left(1-\frac{\delta}{2}\right)
 + \tfrac{\delta}{2}\log\left(\frac{\delta}{2(n-1)}\right),
\end{align*}
achieved at
\begin{align*}
r^*(v,s)
=
\begin{cases}
1-\delta/2, & s=v,\\
\delta/(2n-2), & s\neq v,
\end{cases}
\end{align*}
and for any $q \in \cQ(p_0,\delta)$ written as $q = \sum_{i=1}^k \lambda_i \cdot q_i$ for $q_i = p_0 + \frac{\delta}{2}\cdot (\mathbf{e}_{v_i} - \mathbf{e}_{s_i}) $, the optimizer of the inner problem is given by
\begin{align*}
    \sum_{i=1}^k \lambda_i \cdot \left(\frac{\delta}{2} \cdot (\mathbf{e}_{v_i}\mathbf{e}_{s_i}^\top - \mathbf{e}_{s_i}\mathbf{e}_{s_i}^\top) + \diag(p_0)\right).
\end{align*}
Therefore, the optimum value of the original problem is equal to
\begin{align*}
(1-\tfrac{\delta}{2})\log\left(1-\frac{\delta}{2}\right)
 + \tfrac{\delta}{2}\log\left(\frac{\delta}{2(n-1)}\right) + H(p_0).
\end{align*}
In particular, this is achieved at
\begin{align*}
e^*(v,s)
=
\begin{cases}
\frac{1-\delta/2}{p_0(s)}, & s=v,\\
\frac{\delta}{2(n-1)p_0(s)}, & s\neq v,
\end{cases}
\end{align*}
and the optimizer of the inner problem is given in the same way. 
This completes the proof.
\end{proof}

\subsection{Supporting lemma}

\begin{lemma}[Diagonal Dominance]\label{lem:diagonal-dom}
    Let $\cV$ be a the set $\{1,\dots,n\}$. For any matrix $M_0 \in \mathbb{R}^{n\times n}$, there exists a permutation matrix $P \in \R^{n \times n}$ such that the following condition holds for any $v \in \cV$ and permutation $\sigma \in S_n$:
    \begin{align}\label{eq:diagdom}
        \sum_{i=0}^{\kappa_v-1} (M_0P)(\sigma^i(v),\sigma^i(v)) \geq \sum_{i=0}^{\kappa_v-1} (M_0P)(\sigma^i(v),\sigma^{i+1}(v)).
    \end{align}
    where $\kappa_v$ is the minimum positive integer such that $\sigma^{\kappa_v}(v) = v$. 
\end{lemma}

\begin{proof}
Let $\mathcal{P}$ be the set of all permutation matrices and
\begin{align*}
    P^* = \arg \sup_{P \in \mathcal{P}} \tr[M_0 P].
\end{align*}

Then we claim that $P^*$ is the permutation matrix s.t. $M = M_0P^*$ satisfying Eq.~\eqref{eq:diagdom}. 

Denote 
\begin{align*}
    C = \{v, \sigma(v),\dots,\sigma^{\kappa_v-1}(v)\}.
\end{align*}
Suppose for the sake of contradiction that there exists $\sigma(\cdot)$ and $v$ such that 
\begin{align}\label{eq:cycle_not_satisfied}
    \sum_{i=0}^{\kappa_v-1} M(\sigma^i(v),\sigma^i(v)) < \sum_{i=0}^{\kappa_v-1} M(\sigma^i(v),\sigma^{i+1}(v)).
\end{align}
Define permutation
\begin{align*}
    \bar \sigma(u) = \begin{cases}
        \sigma^{-1}(u),&~ u \in C,\\
        u, &~ u \notin C.
    \end{cases}
\end{align*}
and permutation matrix
\begin{align*}
    P_{\bar \sigma} = \begin{pmatrix}
        \mathbf{e}_{\bar \sigma(1)} & \cdots \mathbf{e}_{\bar \sigma(n)}
    \end{pmatrix}^\top
\end{align*}
Then we have
\begin{align*}
    M P_{\bar \sigma}  = &~ \sum_{i=1}^n M(:,i) \cdot \mathbf{e}_{\bar \sigma(i)}^\top\\
    = &~  \sum_{i \notin C} M(:,i) \cdot \mathbf{e}_{\bar \sigma(i)}^\top + \sum_{i \in C} M(:,i) \cdot \mathbf{e}_{\bar \sigma(i)}^\top\\ 
    = &~ M + \sum_{i \in C} M(:,i) \cdot (\mathbf{e}_{\bar \sigma(i)}^\top - \mathbf{e}_i^\top)\\
    = &~ M + \sum_{j=0}^{\kappa_v-1} M(:,\bar \sigma^j(v)) \cdot (\mathbf{e}_{\bar \sigma^{j+1}(v)}^\top - \mathbf{e}_{\bar \sigma^j(v)}^\top ) .
\end{align*}
Notice
\begin{align*}
    \tr[P_{\bar \sigma} M] = &~ \tr[M] + \tr[\sum_{j=0}^{\kappa_v-1} M(:,\bar \sigma^j(v)) \cdot (\mathbf{e}_{\bar \sigma^{j+1}(v)}^\top - \mathbf{e}_{\bar \sigma^j(v)}^\top )]\\
    = &~ \tr[M] + \tr[\sum_{j=0}^{\kappa_v-1} M(:,\bar \sigma^j(v)) \cdot \mathbf{e}_{\bar \sigma^{j+1}(v)}^\top ] - \tr[\sum_{j=0}^{\kappa_v-1} M(:,\bar \sigma^j(v)) \cdot \mathbf{e}_{\bar \sigma^j(v)}^\top ]\\
    = &~ \tr[M] + \sum_{j=0}^{\kappa_v-1} M(\bar \sigma^{j+1}(v), \bar \sigma^j(v)) - \sum_{j=0}^{\kappa_v-1}  M(\bar \sigma^j(v),\bar \sigma^j(v))\\
    = &~ \tr[M] + \sum_{j=0}^{\kappa_v-1} M(\sigma^{j}(v), \sigma^{j+1}(v)) - \sum_{j=0}^{\kappa_v-1}  M(\sigma^j(v),\sigma^j(v))\\
    > &~ \tr[M]
\end{align*}
where the last inequality follows from Eq.~\eqref{eq:cycle_not_satisfied}.

But $\mathcal{P}$ is a group so $P^* P_{\bar \sigma} \in \mathcal{P}$. This contradicts with the definition of $P^*$.
\end{proof}

\begin{lemma}[Optimizer of Inner Problem]\label{lem:inner}
Let $\cV$ be the set $\{1,\dots,n\}$ and $p_0 \in \Delta(\cV)$ be a distribution over $\cV$ such that $\inf_{v \in \cV}p_0(v) > \delta$.  
Let $M: \cV \times \cV \mapsto \R$ be a matrix that satisfies
\begin{align}
    \sum_{i=0}^{\kappa_v-1} M(\sigma^i(v),\sigma^i(v)) 
     \geq  \sum_{i=0}^{\kappa_v-1} M(\sigma^i(v),\sigma^{i+1}(v))
    \label{eq:M-cycle-cond}
\end{align}
for any $v \in \cV$, $\sigma \in S_n$, where $\kappa_v = \inf\{k\ge1:\sigma^{k}(v) = v\}$. 
Fix distinct $a,b\in\cV$ and define
\begin{align*}
    q = p_0 + \frac{\delta}{2}\cdot (\mathbf{e}_a - \mathbf{e}_b).
\end{align*}
Consider the optimization problem
\begin{align*}
    \sup_{w \in \cP(p_0,q)} J(w), 
    \qquad 
    J(w) := \sum_{v,s \in \cV} w(v,s)   M(v,s),
\end{align*}
where
\begin{align*}
    \cP(p_0,q) 
    = \left\{w:\cV\times\cV\to\R_+  \Big|  
        \sum_{s\in\cV} w(v,s) = q(v) \ \forall v\in\cV,~\sum_{v\in\cV} w(v,s) = p_0(s) \ \forall s\in\cV
      \right\}.
\end{align*}
Then there exists a permutation $\sigma \in S_n$ such that 
\begin{align}
    w^* 
    = \frac{\delta}{2} \sum_{i=0}^{\kappa-2} 
        \left(\mathbf{e}_{\sigma^{i}(a)}\mathbf{e}_{\sigma^{i+1}(a)}^\top 
              - \mathbf{e}_{\sigma^{i+1}(a)}\mathbf{e}_{\sigma^{i+1}(a)}^\top\right) 
      + \diag(p_0)
      \label{eq:p-star-form}
\end{align}
is an optimizer of $J$, where $\kappa$ is the minimum positive integer such that $\sigma^{\kappa}(a) = a$.
\end{lemma}

\begin{proof} First, we establish existence of an optimizer. Define $w^{\mathrm{feas}}$ by
\begin{align*}
    w^{\mathrm{feas}}(a,b) = &~ \frac{\delta}{2},\\
    w^{\mathrm{feas}}(v,v) = &~ 
    \begin{cases}
        p_0(b) - \frac{\delta}{2}, & v=b,\\
        p_0(v), & v\neq b,
    \end{cases}\\
    w^{\mathrm{feas}}(v,s) = &~ 0, \quad (v,s)\notin\{(a,b)\}\cup\{(v,v):v\in\cV\}.
\end{align*}
Because $\inf_v p_0(v)>\delta>\delta/2$, we have $w^{\mathrm{feas}}\ge0$.  
Its row sums satisfy
\begin{align*}
    \sum_s w^{\mathrm{feas}}(a,s) = &~ p_0(a)+\frac{\delta}{2} = q(a),\\
    \sum_s w^{\mathrm{feas}}(b,s) = &~ p_0(b)-\frac{\delta}{2} = q(b),\\
    \sum_s w^{\mathrm{feas}}(v,s) = &~ p_0(v) = q(v), \quad v\notin\{a,b\},
\end{align*}
and its column sums satisfy
\begin{align*}
    \sum_v w^{\mathrm{feas}}(v,b) = &~ p_0(b)-\frac{\delta}{2} + \frac{\delta}{2} = p_0(b),\\
    \sum_v w^{\mathrm{feas}}(v,s) = &~ p_0(s), \quad s\neq b.
\end{align*}
Thus $w^{\mathrm{feas}}\in\cP(p_0,q)$. The set $\cP(p_0,q)$ is a nonempty bounded polytope, and $J$ is linear, so there exists
\begin{align*}
    w^{(0)} \in \arg\sup_{p\in\cP(p_0,q)} J(w).
\end{align*}

\medskip
For a feasible $w$, let $G_w$ be the directed graph on vertex set $\cV$ with an edge
$(v,s)$ whenever $v\neq s$ and $w(v,s)>0$.
A directed cycle in $G_w$ is a $k$-tuple $(v_0,\dots,v_{k-1})$ of distinct vertices,
$k\ge2$, such that
\begin{align*}
    w(v_i,v_{i+1})>0,\quad i=0,\dots,k-1,
\end{align*}
with indices understood modulo $k$ (so $v_k=v_0$).

Next, we show that for an optimal distribution $\bar{w}$, the graph $G_{\bar{w}}$ has no directed cycles. Fix any feasible $w$ and such a cycle $(v_0,\dots,v_{k-1})$.
For $\varepsilon>0$ define $\tilde w$ by
\begin{align}
    \tilde w(v_i,v_{i+1}) = &~ w(v_i,v_{i+1}) - \varepsilon,
    &&i=0,\dots,k-1,
    \label{eq:cycle-transform-off}\\
    \tilde w(v_i,v_i)     = &~ w(v_i,v_i) + \varepsilon,
    &&i=0,\dots,k-1,
    \label{eq:cycle-transform-diag}\\
    \tilde w(x,y)         = &~ w(x,y),
    &&(x,y)\notin\{(v_i,v_{i+1}),(v_i,v_i):0\le i\le k-1\}.
    \notag
\end{align}
Choose
\begin{align*}
    0<\varepsilon\le\inf_{0\le i\le k-1} w(v_i,v_{i+1}),
\end{align*}
so that $\tilde w\ge0$.

\emph{Row sums.}
For each $i$,
\begin{align*}
    \sum_{s\in\cV} \tilde w(v_i,s)
    = &~ \left(w(v_i,v_{i+1})-\varepsilon\right)
       + \left(w(v_i,v_i)+\varepsilon\right)
       + \sum_{s\notin\{v_i,v_{i+1}\}} w(v_i,s)\\
    = &~ \sum_{s\in\cV} w(v_i,s).
\end{align*}
All other rows are unchanged, so
\begin{align}
    \sum_{s} \tilde w(v,s) = \sum_s w(v,s) = q(v), \qquad \forall v\in\cV.
    \label{eq:rows-preserved}
\end{align}

\emph{Column sums.}
For each $i$,
\begin{align*}
    \sum_{x\in\cV} \tilde w(x,v_i)
    = &~ \left(w(v_{i-1},v_i)-\varepsilon\right)
       + \left(w(v_i,v_i)+\varepsilon\right)
       + \sum_{x\notin\{v_{i-1},v_i\}} w(x,v_i)\\
    = &~ \sum_{x\in\cV} w(x,v_i),
\end{align*}
where again indices are taken modulo $k$.
All other columns are unchanged, so
\begin{align}
    \sum_v \tilde w(v,s) = \sum_v w(v,s) = p_0(s), \qquad \forall s\in\cV.
    \label{eq:cols-preserved}
\end{align}
Thus $\tilde w\in\cP(p_0,q)$.

\emph{Objective value.}
Only entries on the cycle and the corresponding diagonals change, hence
\begin{align}
    J(\tilde w)-J(w)
    = &~ \sum_{i=0}^{k-1}
       \Bigl[ \tilde w(v_i,v_i)M(v_i,v_i) + \tilde w(v_i,v_{i+1})M(v_i,v_{i+1}) \\
    &~ - w(v_i,v_i)M(v_i,v_i) - w(v_i,v_{i+1})M(v_i,v_{i+1}) \Bigr] \notag\\
    = &~ \varepsilon \sum_{i=0}^{k-1} \left(M(v_i,v_i) - M(v_i,v_{i+1})\right).
    \label{eq:J-diff-cycle}
\end{align}
Let $\sigma\in S_n$ be the permutation whose cycle on the set
$\{v_0,\dots,v_{k-1}\}$ is $(v_0 v_1 \dots v_{k-1})$ and which fixes all
other vertices.
For $v=v_0$ we have $\kappa_v=k$ and
\begin{align*}
    \sigma^i(v_0) = v_i,\quad
    \sigma^{i+1}(v_0)=v_{i+1},\quad i=0,\dots,k-1,
\end{align*}
so by Eq.~\eqref{eq:M-cycle-cond},
\begin{align}
    \sum_{i=0}^{k-1} M(v_i,v_i)
    = &~ \sum_{i=0}^{\kappa_v-1} M(\sigma^i(v_0),\sigma^i(v_0))
    \notag\\
    \geq &~ \sum_{i=0}^{\kappa_v-1} M(\sigma^i(v_0),\sigma^{i+1}(v_0))
     = \sum_{i=0}^{k-1} M(v_i,v_{i+1}).
    \label{eq:cycle-ineq-specialized}
\end{align}
Combining Eq.~\eqref{eq:J-diff-cycle} and Eq.~\eqref{eq:cycle-ineq-specialized} yields
\begin{align}
    J(\tilde w) - J(w)  \ge  0.
    \label{eq:J-nondecrease-cycle}
\end{align}

Now start from the optimizer $w^{(0)}$.  
If $G_{w^{(0)}}$ has no directed cycles, set $\bar w = w^{(0)}$.
Otherwise, choose a directed cycle in $G_{w^{(0)}}$, apply the above
transformation with maximal $\varepsilon$ as chosen above, and obtain
$w^{(1)}\in\cP(p_0,q)$ with $J(w^{(1)})\ge J(w^{(0)})$.
Since $w^{(0)}$ is optimal, $J(w^{(1)})=J(w^{(0)})$, so $w^{(1)}$ is also
optimal.  Moreover, at least one edge of the chosen cycle has
$w^{(1)}(v_i,v_{i+1})=0$.

Iterating this construction, we obtain a sequence of optimal plans
$w^{(m)}$ in $\cP(p_0,q)$ in which the set of off-diagonal edges with
positive mass strictly decreases whenever there is a directed cycle.
Since there are only finitely many off-diagonal entries, this procedure
must terminate.  We thus obtain an optimal plan $\bar w$ such that
$G_{\bar w}$ contains no directed cycle.

\medskip
Define the off-diagonal flow\begin{align*}
    F(v,s) :=
    \begin{cases}
        \bar w(v,s), & v\neq s,\\
        0, & v=s.
    \end{cases}
\end{align*}
For $v\in\cV$ define
\begin{align*}
    \text{out}(v) := \sum_{s\neq v} F(v,s),
    \qquad
    \text{in}(v)  := \sum_{u\neq v} F(u,v).
\end{align*}
Using the marginal constraints of $\bar w$,
\begin{align*}
    q(v) = &~ \sum_s \bar w(v,s)
          = \bar w(v,v) + \text{out}(v),\\
    p_0(v) = &~ \sum_u \bar w(u,v)
            = \bar w(v,v) + \text{in}(v),
\end{align*}
so
\begin{align}
    \text{out}(v) - \text{in}(v)
    = q(v)-p_0(v)
    =
    \begin{cases}
        \frac{\delta}{2}, & v=a,\\ 
        -\frac{\delta}{2}, & v=b,\\
        0, & v\notin\{a,b\}.
    \end{cases}
    \label{eq:flow-balance}
\end{align}
Thus $F$ is a nonnegative flow of value $\delta/2$ from $a$ to $b$ on the
directed acyclic graph $G_{\bar w}$.

We now decompose $F$ into simple $a$–$b$ paths.
Define $F^{(0)}:=F$ and proceed inductively.
Assume $F^{(m)}$ is a nonnegative flow from $a$ to $b$ with balance
equation Eq.~\eqref{eq:flow-balance}.
Since $\text{out}(a)-\text{in}(a)=\delta/2>0$, there exists
$s_1\neq a$ with $F^{(m)}(a,s_1)>0$.
Set $v_0^{(m)}:=a$ and $v_1^{(m)}:=s_1$.

Suppose $v_0^{(m)} = a, \dots, v_j^{(m)}\neq b$ are already recursively constructed such that we use Eq.~\eqref{eq:flow-balance} and nonnegativity to obtain for $j\ge1$,
\begin{align*}
    \text{in}(v_j^{(m)}) \ge F^{(m)}(v_{j-1}^{(m)},v_j^{(m)}) >0.
\end{align*}
For $v_j^{(m)}\notin\{a,b\}$, Eq.~\eqref{eq:flow-balance} gives
$\text{out}(v_j^{(m)}) = \text{in}(v_j^{(m)}) > 0$, so there exists
$v_{j+1}^{(m)}\neq v_j^{(m)}$ with $F^{(m)}(v_j^{(m)},v_{j+1}^{(m)})>0$.
Since $G_{\bar w}$ is acyclic and finite, the sequence
$v_0^{(m)},v_1^{(m)},\dots$ cannot visit a vertex twice; hence the process
must terminate at a vertex with no outgoing edges.  By
\eqref{eq:flow-balance}, such a vertex must satisfy
$\text{out}(v)-\text{in}(v)\le0$.  
Because all intermediate vertices have balance $0$, the only possible
terminal vertex is $b$.
Thus we obtain a simple path
\begin{align*}
    P^{m+1}: a = v_0^{(m)} \to v_1^{(m)} \to \dots \to v_{k_{m}}^{(m)} = b.
\end{align*}

Let
\begin{align*}
    \alpha_{m} := \inf_{0\le i\le k_{m}-1}
                    F^{(m)}(v_i^{(m)},v_{i+1}^{(m)}) >0,
\end{align*}
and define
\begin{align*}
    F^{(m+1)}(v,s)
    := F^{(m)}(v,s)
       - \alpha_{m}\cdot 
         \mathbf{1}\{(v,s)=(v_i^{(m)},v_{i+1}^{(m)}) 
                         \text{ for some }i=0,\dots,k_m-1\}.
\end{align*}
Then $F^{(m+1)}\ge0$ and satisfies the same balance equations Eq.~\eqref{eq:flow-balance}, but has strictly smaller total flow
$\sum_{v,s}F^{(m+1)}(v,s) = \sum_{v,s}F^{(m)}(v,s) - \alpha_{m}$.

Since the total flow is initially $\delta/2$ and decreases by a positive
amount at each step, this procedure terminates after some $L$ steps with
$F^{(L)}\equiv 0$.  
For $\ell=1,\dots,L$, we obtain simple paths
\begin{align*}
    P^\ell: a = v_0^\ell \to v_1^\ell \to \dots \to v_{k_\ell}^\ell=b,
    \qquad \ell=1,\dots,L,
\end{align*}
and coefficients $\alpha_\ell>0$ such that
\begin{align}
    F(v,s)
    = &~ \sum_{\ell=0}^{L-1} \alpha_\ell
       \mathbf{1}\{(v,s)=(v_i^\ell,v_{i+1}^\ell)
                       \text{ for some }i=0,\dots,k_\ell-1\},
       && v\neq s,
       \label{eq:F-decomp}\\
    \sum_{\ell=0}^{L-1} \alpha_\ell = &~ \frac{\delta}{2}.
    \label{eq:alpha-sum}
\end{align}

\medskip

Finally, we prove that the optimal $G_{w^*}$ contains only a single path and derive the closed form solution $w^*$.

For a general feasible $w\in\cP(p_0,q)$, the column constraints give
\begin{align}
    w(v,v)
    = p_0(v) - \sum_{u\neq v} w(u,v), \qquad v\in\cV.
    \label{eq:diag-from-offdiag}
\end{align}
Thus
\begin{align}
    J(w)
    = &~ \sum_{v} w(v,v)M(v,v) + \sum_{u\neq v} w(u,v)M(u,v) \notag\\
    = &~ \sum_v \left(p_0(v)-\sum_{u\neq v}w(u,v)\right)M(v,v)
       +\sum_{u\neq v}w(u,v)M(u,v) \notag\\
    = &~ \underbrace{\sum_v p_0(v)M(v,v)}_{=:J_0}
       + \sum_{u\neq v} w(u,v)\left(M(u,v)-M(v,v)\right).
    \label{eq:J-decomp-general}
\end{align}
For the optimal plan $\bar w$, the off-diagonal part is $F$, so
\begin{align}
    J(\bar w)
    = J_0 + \sum_{u\neq v} F(u,v)\left(M(u,v)-M(v,v)\right).
    \label{eq:J-barp-F}
\end{align}

For each path $P^\ell$, define its \emph{gain}
\begin{align}
    W(P^\ell)
    := \sum_{i=0}^{k_\ell-1}
       \left(
           M(v_i^\ell,v_{i+1}^\ell)
           - M(v_{i+1}^\ell,v_{i+1}^\ell)
       \right).
    \label{eq:path-gain}
\end{align}
Using the decomposition Eq.~\eqref{eq:F-decomp}, we obtain from Eq.~\eqref{eq:J-barp-F}
\begin{align}
    J(\bar w)
    = &~ J_0
       + \sum_{u\neq v}
         \left(
           \sum_{\ell=1}^L \alpha_\ell
           \mathbf{1}\{(u,v)=(v_i^\ell,v_{i+1}^\ell)
                           \text{ for some }i = 0,...,k_\ell-1\}
         \right)
         \left(M(u,v)-M(v,v)\right) \notag\\
    = &~ J_0 + \sum_{\ell=1}^L \alpha_\ell
       \sum_{i=0}^{k_\ell-1}
       \left(
           M(v_i^\ell,v_{i+1}^\ell)
           - M(v_{i+1}^\ell,v_{i+1}^\ell)
       \right) \notag\\
    = &~ J_0 + \sum_{\ell=1}^L \alpha_\ell W(P^\ell).
    \label{eq:J-barp-pathsum}
\end{align}
By Eq.~\eqref{eq:alpha-sum},
\begin{align}
    J(\bar w)
    = J_0 + \frac{\delta}{2} \sum_{\ell=1}^L \theta_\ell W(P^\ell),
    \qquad
    \theta_\ell := \frac{\alpha_\ell}{\delta/2},\quad
    \theta_\ell\ge0,\ \sum_\ell \theta_\ell=1.
    \label{eq:J-barp-convex}
\end{align}

Let
\begin{align}
    W^* := \sup\{W(P): P \text{ is a simple directed path from }a\text{ to }b\}.
    \label{eq:W-star-def}
\end{align}
Since the graph on $\cV$ is finite, the maximum exists.
From Eq.~\eqref{eq:J-barp-convex} we deduce
\begin{align}
    J(\bar w)
    \le J_0 + \frac{\delta}{2} W^*.
    \label{eq:J-barp-upper}
\end{align}

Now fix an arbitrary simple directed path
\begin{align*}
    P: a=u_0 \to u_1 \to \dots \to u_K = b
\end{align*}
with distinct vertices $u_0,\dots,u_K$.
Define $w^P$ by
\begin{align}
    w^P
    := \frac{\delta}{2}\sum_{i=0}^{K-1}
        \left(\mathbf{e}_{u_i}\mathbf{e}_{u_{i+1}}^\top
              - \mathbf{e}_{u_{i+1}}\mathbf{e}_{u_{i+1}}^\top\right)
       + \diag(p_0).
    \label{eq:pP-def}
\end{align}

We first verify $w^P\in\cP(p_0,q)$.

The diagonal entries of $w^P$ are
\begin{align}
    w^P(u_0,u_0) = &~ p_0(u_0), \notag\\
    w^P(u_j,u_j) = &~ p_0(u_j) - \frac{\delta}{2}, \quad 1\le j\le K,
       \label{eq:pP-diag}\\
    w^P(v,v) = &~ p_0(v), \quad v\notin\{u_0,\dots,u_K\}. \notag
\end{align}
Since $\inf_v p_0(v)>\delta$, we have
$p_0(u_j)-\delta/2 > \delta/2>0$ for $1\le j\le K$; thus all diagonals are
nonnegative.
Off-diagonal entries are either $0$ or $\delta/2$, so $w^P\ge0$.

For $v=u_0=a$,
\begin{align*}
    \sum_s w^P(a,s)
    = &~ w^P(a,a) + w^P(a,u_1)
     = p_0(a) + \frac{\delta}{2}
     = q(a).
\end{align*}
For an internal vertex $u_j$ with $1\le j\le K-1$,
\begin{align*}
    \sum_s w^P(u_j,s)
    = &~ w^P(u_j,u_j) + w^P(u_j,u_{j+1})\\
    = &~ \left(p_0(u_j)-\frac{\delta}{2}\right) + \frac{\delta}{2}
     = p_0(u_j) = q(u_j).
\end{align*}
For $v=u_K=b$,
\begin{align*}
    \sum_s w^P(b,s)
    = &~ w^P(b,b)
     = p_0(b)-\frac{\delta}{2}
     = q(b).
\end{align*}
For $v\notin\{u_0,\dots,u_K\}$, the only nonzero entry in row $v$ is the
diagonal, so
\begin{align*}
    \sum_s w^P(v,s) = p_0(v) = q(v).
\end{align*}
Thus the row constraints are satisfied.

For a vertex $u_{i+1}$ on the path (with $0\le i\le K-1$),
\begin{align*}
    \sum_v w^P(v,u_{i+1})
    = &~ w^P(u_{i+1},u_{i+1}) + w^P(u_i,u_{i+1})\\
    = &~ \left(p_0(u_{i+1})-\frac{\delta}{2}\right) + \frac{\delta}{2}
     = p_0(u_{i+1}).
\end{align*}
For $s\notin\{u_1,\dots,u_K\}$, the only nonzero entry in column $s$ is
$w^P(s,s)=p_0(s)$, so
\begin{align*}
    \sum_v w^P(v,s)=p_0(s).
\end{align*}
Hence $w^P\in\cP(p_0,q)$.

From Eq.~\eqref{eq:pP-def} and linearity of $J$,
\begin{align}
    J(w^P) - J_0
    = &~ \frac{\delta}{2}\sum_{i=0}^{K-1}
       \left(
           M(u_i,u_{i+1}) - M(u_{i+1},u_{i+1})
       \right)
     = \frac{\delta}{2} W(P),
    \label{eq:J-pP}
\end{align}
where $W(P)$ is defined as in Eq.~\eqref{eq:path-gain} for this path $P$.

Now choose a path $P^*$ attaining the maximum gain $W(P^*)=W^*$ in Eq.~\eqref{eq:W-star-def}, and set $w^*:=w^{P^*}$.  Then
\begin{align}
    J(w^*)
    = &~ J_0 + \frac{\delta}{2} W^*
     \ge J_0 + \frac{\delta}{2} W(P^\ell)
     \quad\forall \ell,
     \notag\\
    \geq &~ J_0 + \frac{\delta}{2}\sum_{\ell=1}^L \theta_\ell W(P^\ell)
     = J(\bar w)
     \ge J(w),\quad \forall p\in\cP(p_0,q),
    \label{eq:pstar-optimal}
\end{align}
where we used Eq.~\eqref{eq:J-barp-convex} and Eq.~\eqref{eq:J-barp-upper} in the
last line.  
Thus, $w^*$ is an optimizer of $J$ over $\cP(p_0,q)$ with only a single path.

Write the maximizing path as
\begin{align*}
    P^*: a=u_0 \to u_1 \to \dots \to u_K=b.
\end{align*}
Define a permutation $\sigma\in S_n$ by
\begin{align*}
    \sigma(u_i) = &~ u_{i+1}, \quad i=0,\dots,K-1,\\
    \sigma(u_K) = &~ u_0,\\
    \sigma(v) = &~ v, \quad v\notin\{u_0,\dots,u_K\}.
\end{align*}
Then the orbit of $a$ under $\sigma$ is
\begin{align*}
    a, \sigma(a), \dots, \sigma^K(a)
    = u_0,u_1,\dots,u_K,
\end{align*}
and $\sigma^{K+1}(a)=a$, so the minimal positive integer
$\kappa$ with $\sigma^\kappa(a)=a$ is $\kappa=K+1$.
In particular,
\begin{align*}
    \sigma^i(a) = u_i,\quad i=0,\dots,K.
\end{align*}
Therefore, the definition Eq.~\eqref{eq:pP-def} of $w^*$ can be rewritten as
\begin{align*}
    w^*
    = &~ \frac{\delta}{2}\sum_{i=0}^{K-1}
        \left(\mathbf{e}_{u_i}\mathbf{e}_{u_{i+1}}^\top
              - \mathbf{e}_{u_{i+1}}\mathbf{e}_{u_{i+1}}^\top\right)
       + \diag(p_0)\\
    = &~ \frac{\delta}{2}\sum_{i=0}^{\kappa-2}
        \left(\mathbf{e}_{\sigma^i(a)}\mathbf{e}_{\sigma^{i+1}(a)}^\top
              - \mathbf{e}_{\sigma^{i+1}(a)}\mathbf{e}_{\sigma^{i+1}(a)}^\top\right)
       + \diag(p_0),
\end{align*}
which is exactly Eq.~\eqref{eq:p-star-form}.  This completes the proof.
\end{proof}

\begin{lemma}[Optimizer of Middle Problem]\label{lem:middle}
Let $\cV$ be a the set $\{1,\dots,n\}$, $p_0 \in \Delta(\cV)$ be a distribution over $\cV$ such that $\inf_{v \in \cV}p_0(v) > \delta$, and $M: \cV \times \cV \mapsto \R$ be a matrix. Consider the optimization problem
\begin{align*}
    J: \inf_{q \in \cQ(p_0,\delta)} \sup_{w \in \cP(p_0,q)} \sum_{v,s \in \cV} w(v,s) \cdot M(v,s)
\end{align*}
where 
\begin{align*}
    \cQ(p_0,\delta) = &~ \left\{q \in \Delta(\cV): \|q - p_0\|_1 \leq \delta\right\} \\
    \cP(p_0,q) = &~ \left\{w(v,s): \sum_{s}w(v,s) = q(v), \sum_vw(v,s) =  p_0(s)\right\}
\end{align*}
Then there exists an optimal solution $q^*$ of $J$ coming from the set 
\begin{align*}
    \cQ_{\ext}(p_0,\delta) := \left\{p_0 + \frac{\delta}{2}\cdot (\mathbf{e}_a - \mathbf{e}_b): a \neq b \in \cV\right\}.
\end{align*}
\end{lemma}

\begin{proof}
By \Cref{lem:q_concave_vertices}, $\cQ_{\ext}(p_0,\delta)$ is the set of extreme points of the convex polytope $\cQ(p_0,\delta)$. 
Define $J(q) = \sup_{w \in \cP(p_0,q)} \sum_{v,s \in \cV} w(v,s) \cdot M(v,s)$. We show that $J(q)$ is a concave function. Indeed for all $q_1 \neq q_2$, we let 
\begin{align*}
    w_1 := &~ \arg \sup_{w \in \cP(p_0,q_1)} \sum_{v,s \in \cV} w(v,s) \cdot M(v,s)\\
    w_2 := &~ \arg \sup_{w \in \cP(p_0,q_2)} \sum_{v,s \in \cV} w(v,s) \cdot M(v,s).
\end{align*}
We have
\begin{align*}
    J((q_1+q_2)/2) = &~ \sup_{w \in \cP(p_0,(q_1+q_2)/2)} \sum_{v,s \in \cV} w(v,s) \cdot M(v,s) \\
    \geq &~ \sum_{v,s \in \cV} (w_1(v,s) + w_2(v,s)) /2 \cdot M(v,s)\\
    = &~ (J(q_1)+J(q_2))/2,
\end{align*}
where the second step is because the feasibility of $(w_1(v,s) + w_2(v,s))/2$:
\begin{align*}
    \sum_{s}(w_1(v,s) + w_2(v,s))/2 = (q_1(v)+q_2(v))/2.
\end{align*}
Since $J(q)$ is concave over a convex feasible set, its optimizer must be found on the extreme set $\cQ_{\ext}(p_0,\delta)$. This completes the proof.
\end{proof}

\begin{lemma}[Optimal Subpaths]\label{lem:optimalpaths}
Let $\cV$ be a the set $\{1,\dots,n\}$, $p_0 \in \Delta(\cV)$ be a distribution over $\cV$ such that $\inf_{v \in \cV}p_0(v) > \delta$.
Let $M: \cV \times \cV \mapsto \R$ be a matrix that satisfies
\begin{align*}
    \sum_{i=0}^{\kappa_v-1} M(\sigma^i(v),\sigma^i(v)) 
     \geq  \sum_{i=0}^{\kappa_v-1} M(\sigma^i(v),\sigma^{i+1}(v))
\end{align*}
for any $v \in \cV$, $\sigma \in S_n$, where $\kappa_v = \inf\{k\ge1:\sigma^{k}(v) = v\}$. 
Consider the optimization problem
$$J(q) = \sup_{w \in \cP(p_0,q)}\sum_{v,s \in \cV}w(v,s)\cdot M(v,s)$$
where 
$$\cP(p_0,q) = \left\{w(v,s): \sum_{s}w(v,s) = q(v), \sum_vw(v,s) =  p_0(s)\right\}.$$
If there exists $q = p_0 + \frac{\delta}{2}\cdot(\mathbf{e}_a-\mathbf{e}_b)$ for some $a \neq b \in \cV$ and $w^*_q$ such that $w^*_q$ is the optimizer of $J(q)$ and for a permutation $\sigma$ we have that
$$w^*(q) = \frac{\delta}{2} \cdot \sum_{i=0}^{\kappa-2} (\mathbf{e}_{\sigma^{i}(a)}\mathbf{e}_{\sigma^{i+1}(a)}^\top - \mathbf{e}_{\sigma^{i+1}(a)}\mathbf{e}_{\sigma^{i+1}(a)}^\top) + \diag(p_0),$$
where $\kappa$ is the minimum positive integer such that $\sigma^{\kappa}(a) = a$. Then define $c := \sigma^{i}(a)$ and $d := \sigma^{i+1}(a)$ with $i < \kappa$, and let $q' = p_0 + \frac{\delta}{2}\cdot (\mathbf{e}_c - \mathbf{e}_d)$. We have that the $w^*(q')$ that optimizes $J(q')$ is given by
$$w^*(q') = \frac{\delta}{2}\cdot(\mathbf{e}_{c}\mathbf{e}_d^\top - \mathbf{e}_d \mathbf{e}_d^\top) + \diag(p_0).$$
\end{lemma}

\begin{proof}
Let $c = \sigma^i(a)$ and $d = \sigma^{i+1}(a)$. Define the local transport term corresponding to the edge $(c, d)$ as:
$$ T_{c,d} := \frac{\delta}{2} \cdot (\mathbf{e}_c \mathbf{e}_d^\top - \mathbf{e}_d \mathbf{e}_d^\top). $$
Note that the proposed optimizer for the subproblem $J(q')$ is given by $\hat{w} = \diag(p_0) + T_{c,d}$. 

We proceed by contradiction. Assume that $\hat{w}$ is not the optimizer for $J(q')$. Then there exists a feasible transport plan $\tilde{w} \in \cP(p_0, q')$ such that:
$$ \sum_{v,s \in \cV} \tilde{w}(v,s) M(v,s) > \sum_{v,s \in \cV} \hat{w}(v,s) M(v,s). $$
By \Cref{lem:inner}, $\tilde{w}$ can be chosen to take the form of
\begin{align*}
    \tilde{w} = \frac{\delta}{2} \sum_{i=0}^{\bar \kappa-2} 
        \left(\mathbf{e}_{\bar \sigma^{i}(c)}\mathbf{e}_{\sigma^{i+1}(c)}^\top 
              - \mathbf{e}_{\bar \sigma^{i+1}(c)}\mathbf{e}_{\bar \sigma^{i+1}(c)}^\top\right) 
      + \diag(p_0)
\end{align*}
for some $\bar \sigma \in S_n$ and $\bar \kappa = \inf\{k:\bar \sigma^k(c) = c\}$.
It follows that $\inf_{v \in \cV}\left(\tilde{w} - \diag(p_0)\right)(v,v) \geq -\delta/2$ and $\inf_{v \neq s \in \cV}\left(\tilde{w} - \diag(p_0)\right)(v,s) \geq 0$.

Substituting $\hat{w} = \diag(p_0) + T_{c,d}$ into the inequality, we have:
\begin{equation} \label{eq:contra_ineq}
\sum_{v,s \in \cV} (\tilde{w}(v,s) - \diag(p_0)_{v,s}) M(v,s) > \sum_{v,s \in \cV} (T_{c,d})_{v,s} M(v,s).
\end{equation}

Now, consider the global optimizer $w^*(q)$. By the hypothesis, $w^*(q)$ decomposes into a sum of path segments. We can separate the specific term $T_{c,d}$ from the rest of the path:
$$ w^*(q) = \underbrace{\left( \diag(p_0) + \frac{\delta}{2} \sum_{j \neq i}^{\kappa-2} (\mathbf{e}_{\sigma^{j}(a)}\mathbf{e}_{\sigma^{j+1}(a)}^\top - \mathbf{e}_{\sigma^{j+1}(a)}\mathbf{e}_{\sigma^{j+1}(a)}^\top) \right)}_{R} + T_{c,d}, $$
where $R$ represents the flow on the path excluding the step from $d$ to $c$. We construct a new global transport plan $w_{\text{new}}$ by replacing the local step $T_{c,d}$ in $w^*(q)$ with the ``better'' local flow derived from $\tilde{w}$:
$$ w_{\text{new}} := R + (\tilde{w} - \diag(p_0)). $$
Substituting $R = w^*(q) - T_{c,d}$, we get:
$$ w_{\text{new}} = w^*(q) - T_{c,d} + \tilde{w} - \diag(p_0). $$

We verify that $w_{\text{new}}$ is feasible for the original problem $J(q)$:

Since $\tilde{w} \in \cP(p_0, q')$, its row sum is $q' = p_0 + \frac{\delta}{2}(\mathbf{e}_c - \mathbf{e}_d)$. The term $T_{c,d}$ also corresponds to a row marginal shift of $\frac{\delta}{2}(\mathbf{e}_c - \mathbf{e}_d)$. Thus, $w_{\text{new}}$ preserves the row sums of $w^*(q)$, which equal $q$.
 
Both $\tilde{w}$ and $\diag(p_0) + T_{c,d}$ maintain column sums equal to $p_0$. Thus, $w_{\text{new}}$ maintains the column sums of $w^*(q)$, which equal $p_0$.
 
Since $\inf_{v} p_0(v) > \delta$, $\inf_{v} T_{c,d}(v,v) \geq -\delta/2$ , and $\inf_{v} \left(\tilde{w} - \diag(p_0)\right)(v,v) \geq -\delta/2$, we have $\inf_{v} w_{\text{new}}(v,v) \geq 0$. Furthermore, all off-diagonal entries of $w^*(q) - T_{c,d}$ and $\tilde{w} - \diag(p_0)$ are non-negative, thus we conclude that $w_{\text{new}}$ is non-negative.

Finally, we compare the objective value of $w_{\text{new}}$ to $w^*(q)$:
\begin{align*}
J(w_{\text{new}}) = &~ \sum_{v,s} (R_{v,s} + \tilde{w}(v,s) - \diag(p_0)_{v,s}) M(v,s) \\
= &~ \sum_{v,s} R_{v,s} M(v,s) + \sum_{v,s} (\tilde{w}(v,s) - \diag(p_0)_{v,s}) M(v,s) \\
&> \sum_{v,s} R_{v,s} M(v,s) + \sum_{v,s} (T_{c,d})_{v,s} M(v,s) \quad \text{(by Ineq. \ref{eq:contra_ineq})} \\
= &~ J(w^*(q)).
\end{align*}
We have constructed a feasible solution $w_{\text{new}} \in \cP(p_0, q)$ with a strictly higher objective value than $w^*(q)$. This contradicts the optimality of $w^*(q)$. Therefore, $\hat{w}$ must be the optimizer for $J(q')$.
\end{proof}

\begin{corollary}[Equivalence]\label{cor:equiv}
Let $\cV$ be a the set $\{1,\dots,n\}$ and $p_0 \in \Delta(\cV)$ be a distribution over $\cV$ such that $\inf_{v \in \cV}p_0(v) > \delta$. The following problem
\begin{align*}
    \sup_{r} \inf_{q \in \cQ(p_0,\delta)} \sup_{w \in \cP(p_0,q)} &~  \sum_{v,s \in \cV} w(v,s) \cdot \log r(v,s) \\
    s.t. &~ \sum_{s \in \cV} r(v,s) = 1, ~ \forall v \in \cV\\
    &~ r(v,s) \geq 0,~ \forall v,s \in \cV \notag 
\end{align*}
where 
\begin{align*}
    \cQ(p_0,\delta) = &~ \left\{q \in \Delta(\cV): \|q - p_0\|_1 \leq \delta\right\} \\
    \cP(p_0,q) = &~ \left\{w(v,s): \sum_{s \in \cV}w(v,s) = q(v), \sum_{v  \in \cV}w(v,s) =  p_0(s)\right\}
\end{align*}
is equivalent to the following problem
\begin{align*}
    \sup_{r} \inf_{q \in \cQ_{\text{ext}}(p_0,\delta)} \sup_{w \in \cP_{\text{flow}}(p_0,q)} &~  \sum_{v,s \in \cV} w(v,s) \cdot \log r(v,s) \\
    s.t. &~ \sum_{s \in \cV} r(v,s) = 1, ~ \forall v \in \cV\\
    &~ \sum_{i=0}^{n-1} \log r(v,s)(\sigma^i(v),\sigma^i(v)) \geq \sum_{i=0}^{n-1} \log r(v,s)(\sigma^i(v),\sigma^{i+1}(v)), ~\forall \sigma  \in S_n,\\
    &~ r(v,s) \geq 0,~ \forall v,s \in \cV \notag 
\end{align*}
where 
\begin{align*}
    \cQ_{\ext}(p_0,\delta) = &~ \left\{p_0 + \frac{\delta}{2}\cdot (\mathbf{e}_a - \mathbf{e}_b): a \neq b \in \cV\right\} \\
    \cP_{\flow}(p_0,q) = &~ \left\{\frac{\delta}{2} \cdot \sum_{i=0}^{\kappa-2} (\mathbf{e}_{\sigma^{i}(a)}\mathbf{e}_{\sigma^{i+1}(a)}^\top - \mathbf{e}_{\sigma^{i+1}(a)}\mathbf{e}_{\sigma^{i+1}(a)}^\top) + \diag(p_0): a,b \in \cV, \kappa = \inf \{k \in \Z_+: \sigma^k(a) = a\}\right\}
\end{align*}
\end{corollary}

\begin{proof}
Let 
\begin{align*}
    \cR = &~ \left\{r \in \R^{n \times n}: \sum_{s \in \cV} r(v,s) = 1, ~r(v,s) \geq 0,~ \forall v,s \in \cV \right\}\\
    \cR_{\diag} = &~ \cR \cap \left\{r:\sum_{i=0}^{n-1} \log r(v,s)(\sigma^i(v),\sigma^i(v)) \geq \sum_{i=0}^{n-1} \log r(v,s)(\sigma^i(v),\sigma^{i+1}(v)), ~\forall \sigma \in S_n \right\}.
\end{align*}
Combining \Cref{lem:diagonal-dom}, \Cref{lem:inner}, and \Cref{lem:middle}, we have
\begin{align*}
    \sup_{r \in \cR} \inf_{q \in \cQ(p_0,\delta)} \sup_{w \in \cP(p_0,q)} \sum_{v,s \in \cV} w(v,s) \cdot \log r(v,s) = &~ \sup_{r \in \cR_{\diag}} \inf_{q \in \cQ(p_0,\delta)} \sup_{w \in \cP(p_0,q)} \sum_{v,s \in \cV} w(v,s) \cdot \log r(v,s)\\
    = &~ \sup_{r \in \cR} \inf_{q \in \cQ_{\ext}(p_0,\delta)} \sup_{w \in \cP(p_0,q)} \sum_{v,s \in \cV} w(v,s) \cdot \log r(v,s)\\
    = &~ \sup_{r \in \cR} \inf_{q \in \cQ_{\ext}(p_0,\delta)} \sup_{w \in \cP_{\flow}(p_0,q)} \sum_{v,s \in \cV} w(v,s) \cdot \log r(v,s)
\end{align*}
where the first step uses \Cref{lem:diagonal-dom} and the fact that permuting $\cV$ and the columns of $r$ by the same $\sigma \in S_n$ does not change the objective, the second step uses \Cref{lem:middle}, and the last step uses \Cref{lem:inner}. This completes the proof.
\end{proof}

\begin{proposition}[Reformulated Problem]\label{prop:reform-problem}
Let $\cV$ be a the set $\{1,\dots,n\}$ and $p_0 \in \Delta(\cV)$ be a distribution over $\cV$ such that $\inf_{v \in \cV}p_0(v) > \delta$. Consider the following problem
\begin{align}\label{eq:gen_opt_reform}
    \sup_{r} \inf_{q \in \cQ(p_0,\delta)} \sup_{w \in \cP(p_0,q)} &~  \sum_{v,s \in \cV} w(v,s) \cdot \log r(v,s) \\
    s.t. &~ \sum_{s \in \cV} r(v,s) = 1, ~ \forall v \in \cV\\
    &~ r(v,s) \geq 0,~ \forall v,s \in \cV \notag 
\end{align}
where $p_0$ is a fixed distribution over the sample space $\cV$ such that $\inf_{v \in \cV}p_0(v) > \delta$, and
\begin{align*}
    \cQ(p_0,\delta) = &~ \left\{q \in \Delta(\cV): \|q - p_0\|_1 \leq \delta\right\} \\
    \cP(p_0,q) = &~ \left\{w(v,s): \sum_{s \in \cV}w(v,s) = q(v), \sum_{v  \in \cV}w(v,s) =  p_0(s)\right\}.
\end{align*}
Then the optimum value is
\begin{align*}
J^*
= (1-\tfrac{\delta}{2})\log\left(1-\frac{\delta}{2}\right)
 + \tfrac{\delta}{2}\log\left(\frac{\delta}{2(n-1)}\right).
\end{align*}
In particular, this is achieved at
\begin{align*}
r^*_{a,b}(v,s)
=
\begin{cases}
1-\delta/2, & s=v,\\
\delta/(2n-2), & s\neq v,
\end{cases}
\end{align*}
and for any $q \in \cQ(p_0,\delta)$ written as $q = \sum_{i=1}^k \lambda_i \cdot q_i$ for 
$$q_i = p_0 + \frac{\delta}{2}\cdot (\mathbf{e}_{v_i} - \mathbf{e}_{s_i}) \in \cQ_{\ext}(p_0,\delta) := \left\{p_0 + \frac{\delta}{2}\cdot (\mathbf{e}_a - \mathbf{e}_b): a \neq b \in \cV\right\}$$
the optimizer of the inner problem is given by
\begin{align*}
    \sum_{i=1}^k \lambda_i \cdot \left(\frac{\delta}{2} \cdot (\mathbf{e}_{v_i}\mathbf{e}_{s_i}^\top - \mathbf{e}_{s_i}\mathbf{e}_{s_i}^\top) + \diag(p_0)\right).
\end{align*}
\end{proposition}

\begin{proof}

By \Cref{cor:equiv}, WLOG it suffices to consider $q$ in the form of $p_0 + \frac{\delta}{2}\cdot (\mathbf{e}_a - \mathbf{e}_b)$ for $a \neq b \in \cV$ and $p$ written as $\frac{\delta}{2} \cdot \sum_{i=0}^{\kappa-2} (\mathbf{e}_{\sigma^{i}(a)}\mathbf{e}_{\sigma^{i+1}(a)}^\top - \mathbf{e}_{\sigma^{i+1}(a)}\mathbf{e}_{\sigma^{i+1}(a)}^\top) + \diag(p_0)$ for $\sigma \in S_n, \kappa = \inf \{k \in \Z_+: \sigma^k(a) = a\}$ and $a,b \in \cV$.

Define
\begin{align*}
    J(r,w,q) = &~ \sum_{v,s \in \cV} w(v,s) \cdot \log r(v,s)\\
    J^* = &~ \sup_{r} \inf_{q \in \cQ_{\text{ext}}(p_0,\delta)} \sup_{w \in \cP_{\flow}(p_0,q)} J(r,w,q)\\
    R^* = &~ \{r: \inf_{q \in \cQ_{\text{ext}}(p_0,\delta)} \sup_{w \in \cP_{\flow}(p_0,q)} J(r,w,q) = J^*\}\\
    r^* = &~ \arg\sup\left\{\tr[r]: r \in R^* \right\}.
\end{align*}
We say a solution $(\bar w,\bar q)$ is active if 
\begin{align*}
    \bar w = &~ \arg\sup_{w \in \cP_{\flow}(p_0,\bar q)} J(r^*, w,\bar q)\\
    \bar q = &~ \arg\inf_{q \in \cQ_{\text{ext}}(p_0,\delta)} \sup_{w \in \cP_{\flow}(p_0,q)} J(r^*,w,q)\\
    J(\bar w,\bar q) = &~ J^*.
\end{align*}

Fix $v \neq s \in \cV$ and consider the perturbation for sufficiently small $\epsilon > 0$
\begin{align*}
    \bar r(v,v) \leftarrow &~  r^*(v,v) + \epsilon\\
    \bar r(v,s) \leftarrow &~  r^*(v,s) - \epsilon\\
    \bar r(v',s') \leftarrow &~  r^*(v',s'), ~\forall (v',s') \neq (v,s).
\end{align*}
Then $\bar{r}$ cannot be a valid solution since it has higher trace than $r^*$. It follows that
\begin{align*}
    \frac{d J(\bar r,\bar w,\bar q)}{d\epsilon} = \frac{\bar w(v,v)}{r^*(v,v)} - \frac{\bar w(v,s)}{r^*(v,s)} \leq 0, ~\forall \text{ active } (\bar w,\bar q).
\end{align*}
Notice that the derivative $\frac{d J(\bar r, \bar w,\bar q)}{d\epsilon}$ must be either 
\begin{itemize}
    \item No-transport: $\bar w(v,v) = p_0(v), \bar w(v,s) = 0$ so $\frac{d J(\bar r, \bar w,\bar q)}{d\epsilon} = \frac{p_0(v)}{r^*(v,v)} > 0$.
    \item Middle-way: $\bar w(v,v) = p_0(v) - \delta/2, \bar w(v,s) = \delta/2$ so $\frac{d J(\bar r, \bar w,\bar q)}{d\epsilon} = \frac{p_0(v) - \delta/2}{r^*(v,v)} - \frac{\delta/2}{r^*(v,s)}$.
    \item Transport-start: $\bar w(v,v) = p_0(v), \bar w(v,s) = \delta/2$ so $\frac{d J(\bar r, \bar w,\bar q)}{d\epsilon} = \frac{p_0(v)}{r^*(v,v)} - \frac{\delta/2}{r^*(v,s)}$.
    \item Transport-end: $\bar w(v,v) = p_0(v) - \delta/2, \bar w(v,s) = 0$ so $\frac{d J(\bar r, \bar w,\bar q)}{d\epsilon} = \frac{p_0(v) - \delta/2}{r^*(v,v)} > 0$.
\end{itemize}
Since $\frac{d J(\bar r, \bar w,q)}{d\epsilon} < 0$, we can rule out the first and last cases. Now we can summarize that for any $v \neq s \in \cV$ we have either Middle-way: $\bar w(v,v) = p_0(v) - \delta/2, \bar w(v,s) = \delta/2$ or Transport-start: $\bar w(v,v) = p_0(v), \bar w(v,s) = \delta/2$, and in any case $\frac{p_0(v) - \delta/2}{r^*(v,v)}  - \frac{\delta/2}{r^*(v,s)} \leq 0$ holds.

Since $p_0(v) > \delta$, in either Middle-way or Transport-start case we have
\begin{align*}
    0 \geq \frac{p_0(v) - \delta/2}{r^*(v,v)} - \frac{\delta/2}{r^*(v,s)} > \frac{\delta/2}{r^*(v,v)} - \frac{\delta/2}{r^*(v,s)}.
\end{align*}
It follows that $r^*(v,v) > r^*(v,s)$. Since this argument holds for all $v \neq s \in \cV$, $r^*$ must satisfy $r^*(v,v) > r^*(v,s)$ for all $v \neq s \in \cV$. 

Next, we show that all active $\bar q, \bar w$ can be written as $\bar q = p_0 + \frac{\delta}{2}\cdot (\mathbf{e}_v - \mathbf{e}_s), \bar w = \frac{\delta}{2} \cdot (\mathbf{e}_{v}\mathbf{e}_{s}^\top - \mathbf{e}_{s}\mathbf{e}_{s}^\top) + \diag(p_0)$ for some $v\neq s \in \cV$.

In either Middle-way or Transport-start case, there exists $a \neq b \in \cV$ and $\sigma \in S_n$ and $i \in \Z$ such that $\bar q = p_0 + \frac{\delta}{2}\cdot (\mathbf{e}_a - \mathbf{e}_b)$ and $v = \sigma^i(a), s = \sigma^{i+1}(a)$, due to \Cref{lem:inner}. By \Cref{lem:optimalpaths}, with respect to $\bar q = p_0 + \frac{\delta}{2}\cdot(\mathbf{e}_v-\mathbf{e}_s)$, the optimizer of the inner problem 
must be written as 
\begin{align*}
    \frac{\delta}{2}\cdot(\mathbf{e}_v\mathbf{e}_s^\top - \mathbf{e}_s \mathbf{e}_s^\top) + \text{diag}(p_0) = \arg\sup_{w \in \cP_{\flow}(p_0,\bar q)} J(r^*, w,\bar q).
\end{align*}
Since this argument holds for all $v \neq s \in \cV$, we establish a one-to-one correspondence between $\bar q = p_0 + \frac{\delta}{2}\cdot (\mathbf{e}_v - \mathbf{e}_s)$ and $\bar w = \frac{\delta}{2} \cdot (\mathbf{e}_{v}\mathbf{e}_{s}^\top - \mathbf{e}_{s}\mathbf{e}_{s}^\top) + \diag(p_0)$ for all active $\bar q, \bar w$.

We can now explicitly write:
\begin{align*}
    \inf_{q \in \cQ_{\text{ext}}(p_0,\delta)} \sup_{w \in \cP_{\flow}(p_0,q)} J(r^*,w,q) = &~ \inf_{q \in \cQ_{\text{ext}}(p_0,\delta)} \sup_{w \in \cP_{\flow}(p_0,q)} \sum_{v,s \in \cV} w(v,s) \cdot \log r^*(v,s)\\
    = &~ \inf_{v \neq s \in \cV}\sum_{x \in \cV} p_0(x) \log r^*(x,x) - \frac{\delta}{2}\cdot(\log r^*(s,s) - \log r^*(v,s))
\end{align*}
Define $J(r, v,s) = \sum_{x \in \cV} p_0(x) \log r(x,x) - \frac{\delta}{2}\cdot(\log r(s,s) - \log r(v,s))$, we claim that $J(r^*, v,s)$ must be the same for all $v \neq s \in \cV$.
Otherwise suppose
\begin{align*}
    J(r^*, \bar v, \bar s) = \sup_{v,s}J(r^*, v,s) > J^*.
\end{align*}
Set
\begin{align*}
    \bar r(\bar v, \bar v) \leftarrow &~  r^*(\bar v, \bar v) + \epsilon\\
    \bar r(\bar v, \bar s) \leftarrow &~  r^*(\bar v, \bar s) - \epsilon\\
    \bar r(v',s') \leftarrow &~  r^*(v',s'), ~\forall (v',s') \neq (\bar v, \bar s)
\end{align*}
for sufficiently small $\epsilon > 0$. Then for any $(v',s') \neq (\bar v, \bar s)$ we have 
\begin{align*}
    \frac{d J(\bar r, v,s)}{d\epsilon} = \frac{p_0(s) - \delta/2}{r^*(s,s)} \geq 0
\end{align*}
and thus
\begin{align*}
    \inf_{q \in \cQ_{\ext}(p_0,\delta)} \sup_{w \in \cP_{\flow}(p_0,q)} J(r^*,w,q) = \inf_{v \neq s \in \cV, (v,s) \neq (\bar v, \bar s)}J(r^*, v,s) \geq J^*.
\end{align*}
But $\tr[\bar r] > \tr [r^*]$, this is a contradiction.

From the last argument, we know that the optimal solution $r^*$ is of the form
\begin{align*}
r^*_{a,b}(v,s)
=
\begin{cases}
a, & s=v,\\
b, & s\neq v,
\end{cases}
\qquad a\ge 0,\ b\ge 0,\quad a+(n-1)b=1,
\end{align*}
with $a>b$ (strict diagonal dominance). We now optimize over the two parameters $(a,b)$ subject to this constraint:
\begin{align*}
\sup_{a,b}~\Phi(a,b)
\quad\text{s.t. } a+(n-1)b=1,\ a>0,\ b>0,\ a>b.
\end{align*}
straightforward algebra shows the optimal objective value:
\begin{align*}
\Phi(a^*,b^*)
= &~ (1-\tfrac{\delta}{2})\log a^* + \tfrac{\delta}{2}\log b^*\\
= &~ (1-\tfrac{\delta}{2})\log\left(1-\frac{\delta}{2}\right)
 + \tfrac{\delta}{2}\log\left(\frac{\delta}{2(n-1)}\right).
\end{align*}
attained at 
$a^* = 1 - \frac{\delta}{2}$ and $b^* = \frac{\delta}{2(n-1)}.
$
Thus, the optimal value is
\begin{align*}
J^*
= (1-\tfrac{\delta}{2})\log\left(1-\frac{\delta}{2}\right)
 + \tfrac{\delta}{2}\log\left(\frac{\delta}{2(n-1)}\right).
\end{align*}
This completes the proof.

\end{proof}

\begin{lemma}[Row Normalization]\label{lem:normalizationconstant}
Let $\cV$ be a the set $\{1,\dots,n\}$ and $p_0 \in \Delta(\cV)$ be a distribution over $\cV$ such that $\inf_{v \in \cV}p_0(v) > \delta$. Consider the problem
\begin{align*}
    \sup_{e} \inf_{q \in \cQ(p_0,\delta)} \sup_{w \in \cP(p_0,q)} &~  \sum_{v,s \in \cV} w(v,s) \cdot \log e(v,s) \\
    s.t. &~ \sum_{v,s \in \cV} q(v) p_0(s) e(v,s) \leq 1, ~ \forall q \in \cQ\notag\\
    &~ e(v,s) \geq 0,~ \forall v,s \in \cV \notag 
\end{align*}
where 
\begin{align*}
    \cQ(p_0,\delta) = &~ \left\{q \in \Delta(\cV): \|q - p_0\|_1 \leq \delta\right\} \\
    \cP(p_0,q) = &~ \left\{w(v,s): \sum_{s}w(v,s) = q(v), \sum_vw(v,s) =  p_0(s)\right\}
\end{align*}
Now define the kernel matrix $r:\cV \times \cV \mapsto \mathbb{R}$ such that each of its entries are defined as 
$$r(v,s) = \frac{p_0(s)e(v,s)}{A(v)}, \quad \forall v,s \in \cV,$$
where $A(v) := \sum_s p_0(s)e(v,s)$. Then $A^*(v) = \sum_sp_0(s)e^*(v,s) = 1$ at the optimizer $e^*$.
\end{lemma}

\begin{proof}
Let $e(v,s)$ be any feasible solution to the optimization problem. We define the scaling factor of $e$ at node $v$ as:
$$A(v) := \sum_{s \in \cV} p_0(s) e(v,s).$$
Since we are maximizing an objective involving $\log e(v,s)$, we can assume $e(v,s) > 0$ strictly (otherwise the objective is $-\infty$), and consequently $A(v) > 0$.

We can decompose the matrix $e(v,s)$ into a scale-independent "shape" matrix $\bar{e}(v,s)$ and the scaling factors $A(v)$ as follows:
$$e(v,s) = A(v) \cdot \bar{e}(v,s), \quad \text{where } \bar{e}(v,s) = \frac{e(v,s)}{A(v)}.$$
By construction, the normalized matrix $\bar{e}$ satisfies the normalization property:
$$\sum_{s \in \cV} p_0(s) \bar{e}(v,s) = \sum_{s \in \cV} p_0(s) \frac{e(v,s)}{A(v)} = \frac{1}{A(v)} \sum_{s \in \cV} p_0(s) e(v,s) = 1.$$

Now, let us analyze the constraint given in the problem statement. The condition is:
$$\sum_{v,s \in \cV} q(v) p_0(s) e(v,s) \leq 1, \quad \forall q \in \cQ.$$
Substituting the decomposition of $e(v,s)$:
$$\sum_{v \in \cV} q(v) \left( \sum_{s \in \cV} p_0(s) e(v,s) \right) = \sum_{v \in \cV} q(v) A(v) \leq 1, \quad \forall q \in \cQ.$$

Next, we substitute the decomposition into the objective function. Using the property that $w \in \cP(p_0, q)$ implies $\sum_s w(v,s) = q(v)$, we have:
\begin{align*}
\sum_{v,s \in \cV} w(v,s) \log e(v,s) = &~ \sum_{v,s \in \cV} w(v,s) \log \left( A(v) \cdot \bar{e}(v,s) \right) \\
= &~ \sum_{v,s \in \cV} w(v,s) \log \bar{e}(v,s) + \sum_{v,s \in \cV} w(v,s) \log A(v) \\
= &~ \sum_{v,s \in \cV} w(v,s) \log \bar{e}(v,s) + \sum_{v \in \cV} q(v) \log A(v).
\end{align*}
The first term depends only on the normalized shape $\bar{e}$, while the second term depends only on the scaling factors $A(v)$. To maximize the total objective, we must maximize the second term subject to the feasibility constraint derived above.

Consider the term $\sum_{v \in \cV} q(v) \log A(v)$. Since the logarithm is a concave function, we can apply Jensen's inequality:
$$ \sum_{v \in \cV} q(v) \log A(v) \leq \log \left( \sum_{v \in \cV} q(v) A(v) \right). $$
From the feasibility constraint, we know that $\sum_{v \in \cV} q(v) A(v) \leq 1$. Therefore:
$$ \sum_{v \in \cV} q(v) \log A(v) \leq \log(1) = 0. $$

Thus for every $q$,
\begin{align*}
    \sup_{w \in \cP(p_0,q)} \sum_{v,s} w(v,s)\log e(v,s)
     \le 
    \sup_{w \in \cP(p_0,q)} \sum_{v,s} w(v,s)\log \bar e(v,s).
\end{align*}

And the inequality is strict unless $\sum_v q(v)A(v) = 1$ and $A(v)$ is constant on the support of $q$.

Taking the {minimum over $q$}, we conclude:
\begin{align*}
    \inf_q  \sup_p \sum_{v,s} w(v,s)\log e(v,s)
     \le 
    \inf_q  \sup_p \sum_{v,s} w(v,s)\log \bar e(v,s).
\end{align*}

Equality is achieved if and only if $A(v) = 1$ for all $v \in \cV$.
Thus, for any optimal solution $e^*$, the scaling factors must be set to $1$. Consequently:
$$A^*(v) = \sum_{s \in \cV} p_0(s) e^*(v,s) = 1.$$
\end{proof}

%% file: tau_proof.tex
\section{Proof of \Cref{thm:stopping-time}}\label{sec:tau-proof}

\begin{proof}
We prove the first claim: Let $q^*$ and $p^*$ be the solution of the problem in Eq.~\eqref{eq:log-growth} for the e-value $e$. Define the adversary $\cA^*$ that selects $q^t \equiv q^*$ for all $t \in \Z_+$. Then \Cref{thm:log-growth} implies that 
\begin{align*}
\sup_{\cG} \mathbb{E}_{\mu(\cA^*,\cG)}[\log e(v^t,s^t)] = &~
    \sup_{w \in \cP(p_0,q^*)} \sum_{v,s \in \cV} w(v,s) \cdot \log e(v,s) \\
    \leq &~ J^*.
\end{align*}
Applying \Cref{thm:dynamic-hitting-time}, we have 
\begin{align*}
    \inf_{\cG} \liminf_{\alpha \downarrow 0} \frac{\mathbb{E}_{\mu(\cA^*,\cG)}[\tau_\alpha(e)]}{\log(1/\alpha)} = \frac{1}{\sup_{w \in \cP(p_0,q^*)} \sum_{v,s \in \cV} w(v,s) \cdot \log e(v,s)} \geq \frac{1}{J^*}.
\end{align*}
This establishes the first claim.

For the e-value given by
\begin{align*}
    e^*(v^t,s^t) = \begin{cases}
       \frac{1-\delta/2}{p_0(s^t)}, & v^t = s^t,\\
       \frac{\delta}{2(n-1)p_0(s^t)}, & v^t\neq s^t,
    \end{cases}
\end{align*}
\Cref{thm:log-growth} implies that 
\begin{align*}
    \inf_{q \in \cQ(p_0,\delta)} \sup_{w \in \cP(p_0,q)} &~  \sum_{v,s \in \cV} w(v,s) \cdot \log e^*(v,s) \geq J^*.
\end{align*}
It follows that for any adversary $\cA$, there exists a generator $\cG$ such that
\begin{align*}
    \mathbb{E}_{\mu(\cA,\cG)}[\log e^*(v^t,s^t)] \geq J^* ,\quad \forall t \in \Z_+.
\end{align*}
Applying \Cref{thm:dynamic-hitting-time-lower-bound-drift}, we have for any adversary $\cA$
\begin{align*}
    \inf_{\cG} \liminf_{\alpha \downarrow 0} \frac{\mathbb{E}_{\mu(\cA,\cG)}[\tau_\alpha(e^*)]}{\log(1/\alpha)} \leq \frac{1}{J^*}.
\end{align*}
This establishes the second claim.
\end{proof}

\subsection{Useful results}

\begin{theorem}[Dynamic robust sample complexity with converging drift]
\label{thm:dynamic-hitting-time}
Fix a filtered probability space $(\Omega,\mathcal{F},(\mathcal{F}_t)_{t\geq 0},\mathbb{P})$.
Let $(Y_t)_{t\geq 1}$ be a sequence of integrable random variables adapted to $(\mathcal{F}_t)$, and define the partial sums
\begin{align*}
S_t \;:=\; \sum_{i=1}^t Y_i, \qquad t\geq 1,
\end{align*}
with the convention $S_0:=0$.

Assume the following.

\begin{enumerate}
  \item[(A1)] (\emph{Bounded increments}) There exists a constant $M\in(0,\infty)$ such that
  \begin{align*}
    |Y_t|\;\leq\; M \quad \text{almost surely for all } t\geq 1.
  \end{align*}
  In particular, $Y_t \leq M$ almost surely for all $t$.

  \item[(A2)] (\emph{Positive, converging conditional drift}) There exists a deterministic sequence $(J_t)_{t\geq 1}$ and a constant $J_{\inf}>0$ such that
  \begin{align*}
    \mathbb{E}\!\bigl[\,Y_t \,\bigm|\, \mathcal{F}_{t-1}\,\bigr] \;=\; J_t
    \quad\text{almost surely for all } t\geq 1,
  \end{align*}
  and
  \begin{align*}
    J_{\inf} \;\leq\; J_t \;\leq\; J_{\sup} < \infty
    \quad\text{for all } t\geq 1,
  \end{align*}
  for some finite $J_{\sup}$, and moreover
  \begin{align*}
    J_t \;\longrightarrow\; J_\infty \in (0,\infty)
    \quad\text{as } t\to\infty.
  \end{align*}

  \item[(A3)] (\emph{Stopping rule}) For each threshold $B>0$, define the stopping time
  \begin{align*}
    \tau_B \;:=\; \inf\{\,t\geq 1 : S_t \geq B\,\},
  \end{align*}
  with the usual convention $\inf\emptyset := +\infty$.
\end{enumerate}

Then for every $B>0$ the stopping time $\tau_B$ is integrable, and as $B\to\infty$,
\begin{align*}
  \frac{\mathbb{E}[\tau_B]}{B}
  \;\longrightarrow\; \frac{1}{J_\infty}.
\end{align*}

Equivalently, if we define $B_\alpha := \log(1/\alpha)$ and
\begin{align*}
  \tau_\alpha := \inf\{\,t\geq 1 : S_t \geq B_\alpha\,\}, \qquad \alpha\in(0,1),
\end{align*}
then
\begin{align*}
  \lim_{\alpha\downarrow 0}
  \frac{\mathbb{E}[\tau_\alpha]}{\log(1/\alpha)}
  \;=\; \frac{1}{J_\infty}.
\end{align*}
\end{theorem}

\begin{proof}
We break the proof into several steps. The argument is self-contained and uses only basic properties of conditional expectation and stopping times.

\medskip
\noindent We first establish boundedness of $S_{\tau_B}$ on the event that $\tau_B$ is finite. For each fixed $B>0$, let $\tau_B$ be as in (A3). Because $(Y_t)$ is adapted and the condition $\{S_t\geq B\}$ depends only on $Y_1,\dots,Y_t$, $\tau_B$ is a stopping time with respect to $(\mathcal{F}_t)$.

By definition of $\tau_B$,
\begin{align*}
  S_{\tau_B} \geq B
  \quad\text{on the event } \{\tau_B<\infty\},
\end{align*}

On the event $\{\tau_B<\infty\}$, we have $S_{\tau_B-1} < B$
and $Y_{\tau_B}\leq M$ from the bounded increments assumption (A1). Hence
\begin{align*}
  S_{\tau_B} = S_{\tau_B-1} + Y_{\tau_B}
  < B + M
  \quad\text{on } \{\tau_B<\infty\}.
\end{align*}
Combining the two inequalities gives
\begin{equation}
  B \mathbf{1}_{\{\tau_B<\infty\}}
  \leq
  S_{\tau_B}\mathbf{1}_{\{\tau_B<\infty\}}
  <
  (B+M)\mathbf{1}_{\{\tau_B<\infty\}}.
  \label{eq:overshoot-inequality}
\end{equation}

\medskip

Next, we show that $\tau_B$ is integrable and obtain a crude upper bound on its expectation that will be used later.

For $n\in\mathbb{N}$ define the truncated stopping time
\begin{align*}
  \tau_B^{(n)} := \inf\{\tau_B, n\},
\end{align*}

which is integrable for each fixed $n$. On the one hand,
\begin{align*}
  S_{\tau_B^{(n)}}
  =
  \sum_{t=1}^{\tau_B^{(n)}} Y_t
  = \sum_{t=1}^n Y_t \mathbf{1}_{\{\tau_B\geq t\}},
\end{align*}
because the sum stops at $t=\tau_B$ if $\tau_B\leq n$, and otherwise at $t=n$ if $\tau_B>n$. 

Linearity of expectation yields
\begin{equation}
  \mathbb{E}\bigl[S_{\tau_B^{(n)}}\bigr]
  =
  \sum_{t=1}^n \mathbb{E}\bigl[Y_t \mathbf{1}_{\{\tau_B\geq t\}}\bigr].
  \label{eq:Es-tau-n-expansion}
\end{equation}
This exchange of summation and expectation is justified because $Y_t$ is bounded a.s. for all $t$.

Now we use the conditional drift assumption (A2). Because $Y_t$ is $\mathcal{F}_t$-measurable and $\mathcal{F}_{t-1}$-adapted, and $\{\tau_B\geq t\} = \{\tau_B>t-1\}\in\mathcal{F}_{t-1}$ (by the definition of a stopping time), we have
\begin{align*}
  \mathbb{E}\bigl[Y_t \mathbf{1}_{\{\tau_B\geq t\}}\bigr]
  &= \mathbb{E}\Bigl[
       \mathbf{1}_{\{\tau_B\geq t\}}
       \mathbb{E}\bigl[Y_t \bigm|\mathcal{F}_{t-1}\bigr]
     \Bigr] \\
  &= \mathbb{E}\Bigl[
       \mathbf{1}_{\{\tau_B\geq t\}}J_t
     \Bigr] \\
  &= J_t \mathbb{P}(\tau_B\geq t),
\end{align*}
where in the second line we used (A2), and in the third line we used that $J_t$ is deterministic.

Thus from \eqref{eq:Es-tau-n-expansion} we obtain
\begin{equation}
  \mathbb{E}\bigl[S_{\tau_B^{(n)}}\bigr]
  =
  \sum_{t=1}^n J_t\mathbb{P}(\tau_B\geq t).
  \label{eq:Es-tau-n-with-Jt}
\end{equation}

We now lower-bound the right-hand side by using that $J_t\geq J_{\inf}>0$ for all $t$:
\begin{align*}
  \mathbb{E}\bigl[S_{\tau_B^{(n)}}\bigr]
  \geq
  J_{\inf} \sum_{t=1}^n \mathbb{P}(\tau_B\geq t)
  = J_{\inf} \mathbb{E}[\tau_B^{(n)}],
\end{align*}
because
\begin{align*}
  \sum_{t=1}^n \mathbb{P}(\tau_B\geq t)
  = \sum_{t=1}^n \mathbb{E}\bigl[\mathbf{1}_{\{\tau_B\geq t\}}\bigr]
  = \mathbb{E}\Bigl[\sum_{t=1}^n \mathbf{1}_{\{\tau_B\geq t\}}\Bigr]
  = \mathbb{E}[\tau_B^{(n)}].
\end{align*}

Note that, $S_{\tau_B^{(n)}}\leq B+M$ for all $n$, because whenever we stop (either at time $\tau_B$ or at time $n$ before reaching $B$) we cannot exceed $B+M$ by the same argument as in \eqref{eq:overshoot-inequality}. Thus
\begin{align*}
  \mathbb{E}\bigl[S_{\tau_B^{(n)}}\bigr]\leq B+M
  \quad\text{for all }n.
\end{align*}
We therefore have
\begin{align*}
  J_{\inf}\mathbb{E}[\tau_B^{(n)}]
  \leq B+M \quad\text{for all } n.
\end{align*}
Letting $n\to\infty$ and using monotone convergence $\tau_B^{(n)}\uparrow \tau_B$, we obtain
\begin{equation}\label{eq:tau-upper-crude}
  \mathbb{E}[\tau_B] \leq \frac{B+M}{J_{\inf}} < \infty.
\end{equation}
In particular, $\tau_B$ is integrable for every $B>0$.

\medskip

Now that we know $\tau_B$ is integrable, we can safely expand $S_{\tau_B}$ as an infinite sum and swap expectation and summation.

Indeed, we can write
\begin{align*}
  S_{\tau_B}
  = \sum_{t=1}^{\tau_B} Y_t
  = \sum_{t=1}^\infty Y_t\mathbf{1}_{\{\tau_B\geq t\}},
\end{align*}
where the second equality holds because only finitely many terms are non-zero (those with $t\leq\tau_B$). Taking absolute values,
\begin{align*}
  \sum_{t=1}^\infty |Y_t|\mathbf{1}_{\{\tau_B\geq t\}}
  \leq
  \sum_{t=1}^\infty M\mathbf{1}_{\{\tau_B\geq t\}}
  = M\tau_B,
\end{align*}
and $\mathbb{E}[M\tau_B]<\infty$ by \eqref{eq:tau-upper-crude}. Therefore the sum is integrable and Fubini's theorem gives
\begin{equation}
  \mathbb{E}\bigl[S_{\tau_B}\bigr]
  = \sum_{t=1}^\infty \mathbb{E}\bigl[Y_t\mathbf{1}_{\{\tau_B\geq t\}}\bigr].
  \label{eq:Es-tau-expansion}
\end{equation}

Using $\{\tau_B\geq t\}\in\mathcal{F}_{t-1}$ and (A2), we obtain
\begin{align*}
  \mathbb{E}\bigl[Y_t\mathbf{1}_{\{\tau_B\geq t\}}\bigr]
  = \mathbb{E}\Bigl[
      \mathbf{1}_{\{\tau_B\geq t\}}
      \mathbb{E}\bigl[Y_t \bigm| \mathcal{F}_{t-1}\bigr]
    \Bigr]
  = \mathbb{E}\bigl[\mathbf{1}_{\{\tau_B\geq t\}} J_t\bigr]
  = J_t \mathbb{P}(\tau_B\geq t).
\end{align*}
Therefore,
\begin{equation}
  \mathbb{E}\bigl[S_{\tau_B}\bigr]
  = \sum_{t=1}^\infty J_t\mathbb{P}(\tau_B\geq t).
  \label{eq:Es-tau-Jt-representation}
\end{equation}

Recall from \eqref{eq:overshoot-inequality} that
\begin{align*}
  B \leq S_{\tau_B} < B+M
  \quad\text{on } \{\tau_B<\infty\}.
\end{align*}
Since $\mathbb{P}(\tau_B<\infty)=1$, taking expectations gives
\begin{equation}
  B \leq \mathbb{E}\bigl[S_{\tau_B}\bigr] < B+M.
  \label{eq:Es-tau-bounds}
\end{equation}

Combining \eqref{eq:Es-tau-Jt-representation} and \eqref{eq:Es-tau-bounds} yields the key inequality
\begin{equation}
  B
  \leq
  \sum_{t=1}^\infty J_t\mathbb{P}(\tau_B\geq t)
  <
  B+M.
  \label{eq:key-ineq}
\end{equation}

Define the deviation sequence
\begin{align*}
  \Delta_t := J_t - J_\infty, \qquad t\geq 1.
\end{align*}
Then $|\Delta_t|\leq J_{\sup} + J_\infty <\infty$ for all $t$, and by assumption,
\begin{align*}
  \Delta_t \longrightarrow 0
  \quad\text{as } t\to\infty.
\end{align*}

We rewrite the sum in \eqref{eq:key-ineq} as
\begin{align}
  \sum_{t=1}^\infty J_t\mathbb{P}(\tau_B\geq t)
  &= \sum_{t=1}^\infty \bigl(J_\infty + \Delta_t\bigr)\mathbb{P}(\tau_B\geq t)
  \nonumber\\
  &= J_\infty \sum_{t=1}^\infty \mathbb{P}(\tau_B\geq t)
     + \sum_{t=1}^\infty \Delta_t \mathbb{P}(\tau_B\geq t).
  \label{eq:drift-decomposition}
\end{align}

The first sum is simply $J_\infty \mathbb{E}[\tau_B]$, because
\begin{align*}
  \sum_{t=1}^\infty \mathbb{P}(\tau_B\geq t)
  = \sum_{t=1}^\infty \mathbb{E}\bigl[\mathbf{1}_{\{\tau_B\geq t\}}\bigr]
  = \mathbb{E}\Bigl[\sum_{t=1}^\infty \mathbf{1}_{\{\tau_B\geq t\}}\Bigr]
  = \mathbb{E}[\tau_B],
\end{align*}
where the interchange of summation and expectation is justified because
\begin{align*}
  \sum_{t=1}^\infty \mathbf{1}_{\{\tau_B\geq t\}} = \tau_B
  \quad\text{and}\quad \mathbb{E}[\tau_B]<\infty.
\end{align*}

Thus \eqref{eq:drift-decomposition} becomes
\begin{equation}
  \sum_{t=1}^\infty J_t\mathbb{P}(\tau_B\geq t)
  = J_\infty\mathbb{E}[\tau_B]
    + R_B,
  \label{eq:Es-tau-decomposition}
\end{equation}
where we have defined the remainder term
\begin{align*}
  R_B := \sum_{t=1}^\infty \Delta_t\mathbb{P}(\tau_B\geq t).
\end{align*}

Plugging \eqref{eq:Es-tau-decomposition} into \eqref{eq:key-ineq}, we obtain
\begin{equation}
  B
  \leq
  J_\infty\mathbb{E}[\tau_B] + R_B
  <
  B+M.
  \label{eq:key-ineq-with-RB}
\end{equation}

\medskip

We now show that $R_B$ is negligible compared to $B$ as $B\to\infty$.

Fix an arbitrary $\varepsilon>0$. By the convergence $\Delta_t\to 0$, there exists an integer $T=T(\varepsilon)\geq 1$ such that
\begin{align*}
  |\Delta_t| \leq \varepsilon
  \quad\text{for all } t\geq T.
\end{align*}
Also define
\begin{align*}
  C := \sup_{1\leq t<T} |\Delta_t| < \infty.
\end{align*}

Split the sum defining $R_B$ into the first $T-1$ terms and the tail:
\begin{align*}
  R_B
  &= \sum_{t=1}^{T-1} \Delta_t\mathbb{P}(\tau_B\geq t)
     +
     \sum_{t=T}^\infty \Delta_t\mathbb{P}(\tau_B\geq t).
\end{align*}
We bound the absolute value of each part separately.

For the finite part,
\begin{align*}
  \biggl|\sum_{t=1}^{T-1} \Delta_t\mathbb{P}(\tau_B\geq t)\biggr|
  \leq
  \sum_{t=1}^{T-1} |\Delta_t|\mathbb{P}(\tau_B\geq t)
  \leq
  C \sum_{t=1}^{T-1} 1
  = C(T-1)
  \leq C T.
\end{align*}

For the tail $t\geq T$,
\begin{align*}
  \biggl|\sum_{t=T}^\infty \Delta_t\mathbb{P}(\tau_B\geq t)\biggr|
  \leq
  \sum_{t=T}^\infty |\Delta_t|\mathbb{P}(\tau_B\geq t)
  \leq
  \varepsilon \sum_{t=T}^\infty \mathbb{P}(\tau_B\geq t)
  \leq
  \varepsilon\mathbb{E}[\tau_B],
\end{align*}
again using $\sum_{t=1}^\infty \mathbb{P}(\tau_B\geq t) = \mathbb{E}[\tau_B]$ (and dropping the first $T-1$ terms only makes the sum smaller).

Combining both parts, we have the uniform bound
\begin{equation}
  |R_B| \leq C T + \varepsilon\mathbb{E}[\tau_B]
  \quad\text{for all } B>0.
  \label{eq:RB-bound}
\end{equation}

From \eqref{eq:key-ineq-with-RB} we have
\begin{equation}
  J_\infty\mathbb{E}[\tau_B]
  \geq B - R_B.
  \label{eq:lower-step}
\end{equation}
Using $R_B \leq |R_B|$ together with \eqref{eq:RB-bound},
\begin{align*}
  J_\infty\mathbb{E}[\tau_B]
  \geq B - C T - \varepsilon\mathbb{E}[\tau_B].
\end{align*}
Rearranging,
\begin{align*}
  (J_\infty + \varepsilon)\mathbb{E}[\tau_B]
  \geq B - C T,
\end{align*}
so
\begin{equation}
  \mathbb{E}[\tau_B]
  \geq
  \frac{B - C T}{J_\infty + \varepsilon}.
  \label{eq:tau-lower-bound}
\end{equation}

Similarly, from the upper inequality in \eqref{eq:key-ineq-with-RB},
$-R_B \leq |R_B|$, and \eqref{eq:RB-bound},
\begin{align*}
  J_\infty\mathbb{E}[\tau_B] - |R_B|
  < B+M
  \quad\Longrightarrow\quad
  J_\infty\mathbb{E}[\tau_B] - C T - \varepsilon\mathbb{E}[\tau_B]
  < B+M,
\end{align*}
so
\begin{align*}
  (J_\infty - \varepsilon)\mathbb{E}[\tau_B]
  < B + M + C T.
\end{align*}
Because $\varepsilon>0$ is arbitrary, we obtain
\begin{equation}
  \mathbb{E}[\tau_B]
  \leq
  \frac{B + M + C T}{J_\infty - \varepsilon}.
  \label{eq:tau-upper-bound}
\end{equation}

Now divide both \eqref{eq:tau-lower-bound} and \eqref{eq:tau-upper-bound} by $B$:
\begin{align*}
  \frac{\mathbb{E}[\tau_B]}{B}
  \geq
  \frac{1 - (C T)/B}{J_\infty + \varepsilon},
  \qquad
  \frac{\mathbb{E}[\tau_B]}{B}
  \leq
  \frac{1 + (M+C T)/B}{J_\infty - \varepsilon}.
\end{align*}
Letting $B\to\infty$ (so that $(C T)/B \to 0$ and $(M+C T)/B\to 0$) gives
\begin{align*}
  \liminf_{B\to\infty}
  \frac{\mathbb{E}[\tau_B]}{B}
  \geq \frac{1}{J_\infty + \varepsilon},
  \qquad
  \limsup_{B\to\infty}
  \frac{\mathbb{E}[\tau_B]}{B}
  \leq \frac{1}{J_\infty - \varepsilon}.
\end{align*}
Since $\varepsilon>0$ was arbitrary, we may let $\varepsilon\downarrow 0$ to obtain
\begin{align*}
  \liminf_{B\to\infty}
  \frac{\mathbb{E}[\tau_B]}{B}
  \geq \frac{1}{J_\infty},
  \qquad
  \limsup_{B\to\infty}
  \frac{\mathbb{E}[\tau_B]}{B}
  \leq \frac{1}{J_\infty}.
\end{align*}
Hence the limit exists and equals $1/J_\infty$:
\begin{align*}
  \lim_{B\to\infty}
  \frac{\mathbb{E}[\tau_B]}{B}
  = \frac{1}{J_\infty}.
\end{align*}

Finally, choosing $B=B_\alpha := \log(1/\alpha)$ for $\alpha\in(0,1)$ yields
\begin{align*}
  \lim_{\alpha\downarrow 0}
  \frac{\mathbb{E}[\tau_\alpha]}{\log(1/\alpha)}
  = \frac{1}{J_\infty},
\end{align*}
which completes the proof.
\end{proof}

\begin{theorem}[Dynamic hitting-time upper bound with converging lower drift]
\label{thm:dynamic-hitting-time-lower-bound-drift}
Fix a filtered probability space
$(\Omega,\mathcal{F},(\mathcal{F}_t)_{t\geq 0},\mathbb{P})$.
Let $(Y_t)_{t\geq 1}$ be a sequence of integrable random variables adapted to
$(\mathcal{F}_t)$, and define the partial sums
\begin{align*}
S_t := \sum_{i=1}^t Y_i, \qquad t \geq 1,
\end{align*}
with the convention $S_0 := 0$.

Assume the following.

\begin{enumerate}
\item[(A1)] (\emph{Bounded increments})
There exists a constant $M \in (0,\infty)$ such that
\begin{align*}
|Y_t| \leq M \quad \text{almost surely for all } t \geq 1 .
\end{align*}

\item[(A2$\ge$)] (\emph{Positive, converging lower conditional drift})
There exists a deterministic sequence $(J_t)_{t\geq 1}$ and constants
$0 < J_{\inf} \leq J_{\sup} < \infty$ such that
\begin{align*}
\mathbb{E}[Y_t \mid \mathcal{F}_{t-1}] \geq J_t
\quad \text{almost surely for all } t \geq 1 ,
\end{align*}
and
\begin{align*}
J_{\inf} \leq J_t \leq J_{\sup} \quad \text{for all } t \geq 1 ,
\end{align*}
with
\begin{align*}
J_t \longrightarrow J_\infty \in (0,\infty)
\quad \text{as } t \to \infty .
\end{align*}

\item[(A3)] (\emph{Stopping rule})
For each threshold $B > 0$, define the stopping time
\begin{align*}
\tau_B := \inf\{ t \geq 1 : S_t \geq B \},
\end{align*}
with the convention $\inf\emptyset := +\infty$.
\end{enumerate}

Then for every $B > 0$, the stopping time $\tau_B$ is integrable. Moreover,
for every $\varepsilon \in (0, J_\infty)$ there exists a finite constant
$C_\varepsilon$ such that
\begin{align*}
\mathbb{E}[\tau_B]
\leq \frac{B + M + C_\varepsilon}{J_\infty - \varepsilon}
\quad \text{for all } B > 0 .
\end{align*}
Consequently,
\begin{align*}
\limsup_{B \to \infty}
\frac{\mathbb{E}[\tau_B]}{B}
\leq \frac{1}{J_\infty} .
\end{align*}

Equivalently, if $B_\alpha := \log(1/\alpha)$ and
\begin{align*}
\tau_\alpha := \inf\{ t \geq 1 : S_t \geq B_\alpha \},
\qquad \alpha \in (0,1),
\end{align*}
then
\begin{align*}
\limsup_{\alpha \downarrow 0}
\frac{\mathbb{E}[\tau_\alpha]}{\log(1/\alpha)}
\leq \frac{1}{J_\infty} .
\end{align*}
\end{theorem}

\begin{proof}

Define the deterministic partial sums
\begin{align*}
A_t := \sum_{i=1}^t J_i, \qquad A_0 := 0,
\end{align*}
and the excess process
\begin{align*}
Z_t := S_t - A_t = \sum_{i=1}^t (Y_i - J_i), \qquad Z_0 := 0.
\end{align*}
By assumption (A2$\ge$),
\begin{align*}
\mathbb{E}[Z_t \mid \mathcal{F}_{t-1}]
= Z_{t-1} + \mathbb{E}[Y_t - J_t \mid \mathcal{F}_{t-1}]
\geq Z_{t-1}
\quad \text{a.s.},
\end{align*}
so $(Z_t)$ is a submartingale.

For $n \in \mathbb{N}$ define the bounded stopping time
\begin{align*}
\tau_B^{(n)} := \tau_B \wedge n .
\end{align*}

We claim that for every $n$,
\begin{align*}
S_{\tau_B^{(n)}} \leq B + M \quad \text{a.s.}
\end{align*}
Indeed, on $\{\tau_B \leq n\}$ we have $\tau_B^{(n)} = \tau_B$ and
$S_{\tau_B - 1} < B$ by definition of $\tau_B$, while $Y_{\tau_B} \leq M$ a.s.,
hence
\begin{align*}
S_{\tau_B} = S_{\tau_B - 1} + Y_{\tau_B} < B + M .
\end{align*}
On $\{\tau_B > n\}$ we have $\tau_B^{(n)} = n$ and $S_n < B$, so again
$S_{\tau_B^{(n)}} < B + M$.
This proves the claim.

Next, we show that
\begin{align*}
\mathbb{E}[Z_{\tau_B^{(n)}}] \geq 0,
\end{align*}
or equivalently,
$\mathbb{E}[S_{\tau_B^{(n)}}] \geq \mathbb{E}[A_{\tau_B^{(n)}}].$

Since $\tau_B^{(n)} \leq n$,
\begin{align*}
Z_{\tau_B^{(n)}}
&= \sum_{t=1}^{\tau_B^{(n)}} (Y_t - J_t)
= \sum_{t=1}^n \mathbf{1}\{\tau_B^{(n)} \geq t\} (Y_t - J_t).
\end{align*}
For $t \leq n$, $\{\tau_B^{(n)} \geq t\} = \{\tau_B \geq t\} \in \mathcal{F}_{t-1}$,
since $\tau_B$ is a stopping time. Therefore,
\begin{align*}
\mathbb{E}\!\left[\mathbf{1}\{\tau_B^{(n)} \geq t\}(Y_t - J_t)\right]
&= \mathbb{E}\!\left[
\mathbb{E}\!\left[
\mathbf{1}\{\tau_B^{(n)} \geq t\}(Y_t - J_t)
\mid \mathcal{F}_{t-1}
\right]\right] \\
&= \mathbb{E}\!\left[
\mathbf{1}\{\tau_B^{(n)} \geq t\}
\mathbb{E}[Y_t - J_t \mid \mathcal{F}_{t-1}]
\right] \\
&\geq 0 ,
\end{align*}
where the inequality follows from assumption (A2$\ge$).
Summing over $t = 1,\dots,n$ yields $\mathbb{E}[Z_{\tau_B^{(n)}}] \geq 0$.

Thus, we have that
\begin{align*}
\mathbb{E}[A_{\tau_B^{(n)}}]
\leq \mathbb{E}[S_{\tau_B^{(n)}}]
\leq B + M
\quad \text{for all } n .
\end{align*}
Since $J_t \geq J_{\inf} > 0$, the sequence $(A_t)$ is increasing and
$A_t \to \infty$ as $t \to \infty$. Because $\tau_B^{(n)} \uparrow \tau_B$,
the monotone convergence theorem yields
\begin{align*}
\mathbb{E}[A_{\tau_B}]
= \lim_{n \to \infty} \mathbb{E}[A_{\tau_B^{(n)}}]
\leq B + M .
\end{align*} 
Moreover, since
$A_{\tau_B} \geq J_{\inf} \tau_B$,
\begin{align*}
J_{\inf} \mathbb{E}[\tau_B] \leq \mathbb{E}[A_{\tau_B}] \leq B + M ,
\end{align*}
so $\tau_B$ is integrable.

Fix $\varepsilon \in (0, J_\infty)$. Since $J_t \to J_\infty$, there exists
$N = N(\varepsilon)$ such that
\begin{align*}
J_t \geq J_\infty - \varepsilon \quad \text{for all } t \geq N .
\end{align*}
Define the finite constant
\begin{align*}
C_\varepsilon
:= \sup_{0 \leq t \leq N-1} \bigl((J_\infty - \varepsilon)t - A_t\bigr).
\end{align*}
Then for all $t \geq 0$,
\begin{align*}
A_t \geq (J_\infty - \varepsilon)t - C_\varepsilon .
\end{align*}

Applying the bound from in the previous display at the random time $\tau_B$ and taking expectations yields
\begin{align*}
\mathbb{E}[A_{\tau_B}]
\geq (J_\infty - \varepsilon)\mathbb{E}[\tau_B] - C_\varepsilon .
\end{align*}
Combining this with $\mathbb{E}[A_{\tau_B}] \leq B + M$ yields
\begin{align*}
(J_\infty - \varepsilon)\mathbb{E}[\tau_B]
\leq B + M + C_\varepsilon ,
\end{align*}
and therefore
\begin{align*}
\mathbb{E}[\tau_B]
\leq \frac{B + M + C_\varepsilon}{J_\infty - \varepsilon} .
\end{align*}
Dividing by $B$ and letting $B \to \infty$, then letting $\varepsilon \downarrow 0$,
gives
\begin{align*}
\limsup_{B \to \infty} \frac{\mathbb{E}[\tau_B]}{B}
\leq \frac{1}{J_\infty} .
\end{align*}
This completes the proof.
\end{proof}

%% file: claims.tex
\section{Useful Claims}

\begin{lemma}\label{lem:q_concave_vertices}
Let $\cV$ be the set $\{1,\dots,n\}$ and $p_0 \in \Delta(\cV)$ be a distribution over $\cV$ such that $\min_{v \in \cV}p_0(v) > \delta$. 
Define
\begin{align*}
    \cQ(p_0,\delta) = \left\{q \in \Delta(\cV): \|q - p_0\|_1 \leq \delta\right\} 
\end{align*}
Then $\cQ(p_0,\delta)$ is a convex polytope whose vertex set is given by:
$$\cQ_{\ext}(p_0,\delta) := \{p_0 + (\mathbf{e}_i-\mathbf{e}_j)\cdot \delta/2: (i,j) \in \cV\times \cV, i\neq j\}.$$
\end{lemma}

\begin{proof}

Recall $\cQ(p_0,\delta) := \{q \in \Delta^n: \|q - p_0\|_{\ell_1} \leq \delta\}$. First, we note that $\cQ(p_0,\delta)$ is the intersection of two convex polytopes, hence it must also be a convex polytope. We claim that $\cQ(p_0,\delta) = \mathrm{Conv}(\cQ_{\ext}(p_0,\delta))$. Suppose $q \in \cQ(p_0,\delta)$, then we can write that 
$q_i - p_i  = s_i$
for $s_i \in [-\delta/2,\delta/2]$, $\sum s_i = 0$, and $\sum \lvert s_i\rvert \leq \delta$. Next, suppose for contradiction that $s_j > \delta/2$, then $\sum_{i \neq j}s_i = -s_j$ and hence,
$$\sum_{k}\lvert s_k\rvert \geq \lvert s_j\rvert + \lvert\sum_{i \neq j}s_i\rvert > \delta,$$
which violates the TV constraint. 

Now we show: for any $q \in \cQ(p_0,\delta)$,
there exist nonnegative weights $\lambda_{ij}$ summing to $1$ such that
\begin{align*}
    q = \sum_{i \neq j} \lambda_{ij} \left( p_0 + \frac{\delta}{2}(\mathbf{e}_i - \mathbf{e}_j) \right).
\end{align*}

We construct the decomposition as follows:

\begin{itemize}
    \item Let $P = \{i : s_i > 0\},\, N = \{j : s_j < 0\}$.
    \item Necessarily $\sum_{i \in P} s_i = - \sum_{j \in N} s_j 
    = \frac{1}{2} \sum_k |s_k| \le \delta/2$.
    \item Define nonnegative coefficients $\alpha_{ij}$ for $i \in P,\, j \in N$
    such that
    \begin{align*}
        \sum_{j \in N} \alpha_{ij} = \frac{s_i}{\delta/2}, 
        \qquad
        \sum_{i \in P} \alpha_{ij} = \frac{-s_j}{\delta/2}.
    \end{align*}
\end{itemize}

The existence follows from Hoffman’s circulation theorem.

\begin{itemize}
    \item Then set $\lambda_{ij} = \alpha_{ij}$. Summing,
    \begin{align*}
        p_0 + \frac{\delta}{2} \sum_{i \neq j} \lambda_{ij}(\mathbf{e}_i - \mathbf{e}_j)
        = p_0 + s = q.
    \end{align*}
\end{itemize}

This shows $q$ lies in the convex hull of the $v_{ij}$.

It is now enough to show that any $v \in \cQ_{\ext}(p_0,\delta)$ cannot be generated by a convex combination of two distinct points in $\cQ(p_0,\delta)$ which will prove that $v \in \cQ_{\ext}(p_0,\delta)$ is a vertex of $\cQ(p_0,\delta)$ and hence, $\cQ(p_0,\delta)$ is generated by the convex hull of $\cQ_{\ext}(p_0,\delta)$. Suppose for contradiction that this is the case. Then there exists $q,q'\in\cQ(p_0,\delta)$ and $\lambda \in (0,1)$ such that 
$$p_0 + (\mathbf{e}_i-\mathbf{e}_j)\frac{\delta}{2} = \lambda q + (1-\lambda)q'.$$
Let $q = p + \eta\frac{\delta}{2}$ and $q' = p + \eta'\frac{\delta}{2}$ where $\eta \neq \eta'$ then we have that 
$$\mathbf{e}_i - \mathbf{e}_j = \lambda \eta + (1-\lambda)\eta'.$$
Then we must have that 
$$\lambda \eta_i + (1-\lambda)\eta'_i = 1, \quad \lambda \eta_{j}+(1-\lambda)\eta'_j = -1,$$
and
$$\lambda \eta_{k} + (1-\lambda)\eta'_k  = 0.$$
But for either of the first two conditions to be true, we must have that $\eta_i = \eta_i' = 1$ and $\eta_j = \eta_j' = -1$ Since all elements $\eta_k$ and $\eta_k'$ are bounded in magnitude by $1$. Furthermore, because $\lvert \eta_i\rvert + \lvert \eta_j\rvert = 2$, all other $\eta_k = 0$ (same for $\eta'$). Thus, $\eta = \eta'$ which is the desired contradiction.
\end{proof}

%% file: add_experiments.tex
\section{Additional Experiments}
\label{sec:add_exp}
This section summarizes the complete experiment results on the \textsc{MarkMyWords} benchmark~\citep{piet2023mark} under three temperature settings ($0.3, 0.7, 1.0$). For the generation scheme, we follow the hyperparameter configurations in \citet{huang2025watermarking} with $K=20, |\Omega_h| = 2, \delta = 0.3$. Table~\ref{tab:comparison-quality} reports generation quality across watermarking schemes. Overall, quality is relatively stable across schemes, and our method achieves quality that is comparable to (and in some cases close to the best among) the baselines at each temperature. Table~\ref{tab:comparison-size} reports detection size, measuring the average number of tokens needed to detect the watermark under a range of perturbations, where lower is better. Here, our scheme consistently attains the smallest size at every temperature, substantially outperforming prior methods. This indicates that our detector can reliably identify watermarked text using fewer tokens. Together, these results confirm that our e-value-based scheme improves detection efficiency without sacrificing generation quality.

\begin{table}[h]
    \centering
    \begin{NiceTabular}{c|cccccc}
        \toprule
        {\bf Temp.} & {\bf Exponential} & {\bf Inverse Transform} & {\bf Binary} & {\bf Distribution Shift} & {\bf SEAL} & {\bf \Cref{thm:log-growth}} \\
        \midrule
        0.3 & 0.906 & {\bf  0.910} & 0.905 & 0.900 & 0.905 & { 0.909} \\
        0.7 & 0.907 & 0.917 & {\bf 0.919} & 0.912 & 0.901 & {\bf 0.919} \\
        1.0 & 0.898 & {\bf  0.917 } & 0.905 & 0.907 & 0.871 & { 0.902} \\
        \bottomrule
    \end{NiceTabular}
    \caption{{\fontfamily{ptm}\selectfont Quality ($\uparrow$) across different temperature settings. Best (per temperature) is highlighted in bold.}}
    \label{tab:comparison-quality}
\end{table}

\begin{table}[h]
    \centering
    \begin{NiceTabular}{c|cccccc}
        \toprule
        {\bf Temp.} & {\bf Exponential} & {\bf Inverse Transform} & {\bf Binary} & {\bf Distribution Shift} & {\bf SEAL} & {\bf \Cref{thm:log-growth}} \\
        \midrule
        0.3 & $\infty$ & $\infty$ & $\infty$ & 112.5 & 106.0 & {\bf 97.0} \\
        0.7 & $\infty$ & 734.0 & $\infty$ & 145.0 & 84.5 & {\bf 72.0} \\
        1.0 & 240.5 & 163.5 & $\infty$ & 317.0 & 133.0 & {\bf 97.5} \\
        \bottomrule
    \end{NiceTabular}
    \caption{{\fontfamily{ptm}\selectfont Size ($\downarrow$) across different temperature settings. Best (per temperature) is highlighted in bold.}}
    \label{tab:comparison-size}
\end{table}